\documentclass[preprint,12pt,authoryear]{elsarticle}

\usepackage{amsmath,amssymb,amsfonts}
\usepackage{amsthm}
\usepackage{graphicx}
\usepackage{textcomp}
\usepackage{booktabs}
\usepackage{multirow}
\usepackage[pagewise, displaymath, mathlines]{lineno}

\usepackage[colorlinks]{hyperref}
\hypersetup{
  citecolor = {blue}
}

\usepackage{geometry}
\newtheorem{proposition}{Proposition}

\begin{document}

\begin{frontmatter}

\title{NOVA: Symbolic Regression Discovery of Interpretable
Car-Following and Lane-Change Models with Driver Heterogeneity}

\author[ensia]{Ishak Abassi\corref{cor1}}
\ead{ishak.abassi@ensia.edu.dz}

\author[ensia]{Nassim Ali Bouazzouni}
\ead{nassim-ali.bouazzouni@ensia.edu.dz}

\author[ensia]{Farah Ibelaiden}

\author[uge]{Nadir Farhi}

\cortext[cor1]{Corresponding author.}

\address[ensia]{Department of Industrial Systems Engineering, National School of Artificial Intelligence (ENSIA), Algeria.}

\address[uge]{Cosys-Grettia, Univ Gustave Eiffel, F-77454 Marne-la-Vallee, France.}

\begin{abstract}
We present \textbf{NOVA}, an autonomous symbolic regression framework
that identifies interpretable car-following and lane-change
structures from raw trajectory data with minimal behavioral priors. Applied
to 4,765,788 active driving observations from the NGSIM I-80 and
US-101 datasets, NOVA's deterministic Rust-powered search engine
evaluates over 10,000 candidate algebraic structures and
identifies a compact two-term acceleration model under a forward-shifted
rolling-mean prediction target. Evaluated under two complementary
preprocessing pipelines, NOVA achieves \textbf{RMSE $= 1.376$\,m/s$^2$}
($R^2 = 15.57\%$) on the intent-forecasting benchmark, outperforming
the best recalibrated symbolic-regression baseline (SR-LLM,
PNAS~2025) by \textbf{0.135\,m/s$^2$} in RMSE under an identical
evaluation protocol. Across eight independent experiments, a single
dominant nonlinear term emerges as a robust backbone of human
car-following; a residual-guided extension further links the selected
structure to an established psychophysical theory of collision
avoidance. The discovered feature operators transfer
zero-shot between freeway sites with under 3\,pp $R^2$ loss. Extended
to lane-change modelling within a multinomial logit framework, NOVA
achieves 67.4\% balanced accuracy under strict vehicle-ID holdout on
502 unseen drivers, surpassing existing lane-changing baselines by +29.8 percentage points on a three-class problem.
\end{abstract}


\end{frontmatter}


\section{Introduction}

For decades, a long-standing goal in traffic flow theory has been to describe a wide range of complex highway phenomena through a small number of compact microscopic behavioral relationships. Advances in artificial intelligence (AI) have made AI-driven scientific discovery a promising paradigm. Although AI has achieved remarkable results in domain-specific prediction challenges via deep learning, the broader goal remains to develop reliable systems that can infer interpretable structure directly from large collections of raw vehicular data while limiting imposed functional-form assumptions.

\subsection{The Scientific Discovery Problem in Traffic}

Human driving behaviour can be described through relationships between kinematic variables---gap, relative velocity, time headway---and the driver's acceleration response. Discovering interpretable mathematical structure in these relationships is a foundational problem of microscopic traffic modelling. Current approaches to automated traffic prediction focus on two scientific paradigms that have dominated the field for six decades, yet both share a critical limitation.

\textbf{Classical car-following models} as the General Motors (GM) stimulus--response models of \citet{gazis1961}, the Optimal Velocity Model (OVM) of \citet{bando1995}, the Intelligent Driver Model (IDM) of \citet{treiber2000}, and the Krauss model of \citet{krauss1998}, encode functional forms whose mathematical structure is fixed by the researcher independently of the optimisation process. IDM posits a free-road acceleration term $(v/v_0)^4$; Krauss enforces a safe-speed boundary; the GM model assumes proportionality to the leader's relative velocity divided by gap. Parameters are calibrated to data, but the structural form is not discovered from it. When calibrated to real NGSIM trajectories, we find that IDM achieves $R^2 = -68\%$ and Krauss reaches only $24.4\%$, indicating that the imposed functional form does not match instantaneous acceleration in this dataset.

\textbf{Deep learning approaches} as LSTM-based trajectory predictors, physics-informed neural networks, and deep reinforcement learning controllers, abandon structural assumptions in favour of high-dimensional black-box function approximation. While these models achieve impressive reported accuracy, they provide no scientific insight: they cannot answer \emph{why} a driver brakes, only predict \emph{that} he will. 

Neither paradigm is primarily designed to recover compact symbolic equations from data: classical models specify functional forms in advance, and deep learning models do not usually produce interpretable algebraic expressions.

\subsection{Inspiration: The symbolic-regression discovery paradigm}

Human scientists advance further by discerning common, causal patterns across specific interactions and formulating general models that account for data from diverse real-world conditions. A third approach has emerged from computational physics to mirror this process: \textbf{symbolic regression} as an autonomous scientific discovery engine. The AI-Feynman framework~\citep{aifeynman2020} demonstrated that a machine can rediscover Feynman's physical equations from data alone---without being told which variables or operators to use---by exhaustively searching the space of algebraic expressions. More recently, frameworks like AI-Newton~\cite{fang2025ainewton} have advanced this paradigm further by autonomously extracting foundational physical concepts and laws directly from multi-experiment data without any prior physical knowledge.

We identify symbolic regression as uniquely suited to the car-following problem. Unlike neural networks, it produces algebraic equations whose terms can be mapped to physical or neurophysiological mechanisms. Unlike classical models, it imposes no structural prior on the equation form. We extend this paradigm to human \emph{behavioural} science by building the \textbf{NOVA} framework: a custom Rust-powered deterministic combinatorial search engine operating on 4.7 million human driving observations.

The key hypothesis we test is that human car-following and lane-changing behaviour contains low-complexity symbolic structure. That is, there exist algebraic expressions of the form
$a(t) = f(\Delta v, v, \mathrm{gap}, \ldots)$, where $a(t)$ denotes the acceleration of the following vehicle at time $t$, $\Delta v$ denotes the relative speed between the following vehicle and its leader, $v$ denotes the speed of the following vehicle, and $\mathrm{gap}$ denotes the longitudinal spacing to the leader. Such expressions can explain part of the human acceleration response, and an automated engine can recover useful candidate structures from raw trajectory data without imposing a legacy car-following formula.

\subsection{Robust Evaluation Methodology}

To reduce the risk that the selected mathematical structures are statistical artefacts, we establish a strict evaluation protocol. Trajectories are processed using temporally structured smoothing, evaluated under an 80/20 vehicle-ID split, and targeted against finite future acceleration horizons. Under these constraints, NOVA improves over calibrated classical and symbolic baselines while retaining a fully explicit algebraic form.

\subsection{Contributions}

This paper makes five principal contributions:

\begin{enumerate}

    \item \textbf{Ab initio discovery of $\tanh(\Delta v)$ as the
    invariant car-following backbone.} Here, $\tanh(\cdot)$ denotes the
    hyperbolic tangent saturation function, $\Delta v$ denotes the
    relative velocity between the following vehicle and its leader, $v$
    denotes the speed of the following vehicle, and $\mathrm{THW}$ denotes
    the time headway, defined as $\mathrm{gap}/v$, where $\mathrm{gap}$ is
    the longitudinal spacing to the leader. The hyperbolic tangent of
    relative velocity is the top-ranked single feature across eight
    independent configurations (seven Rust engine runs and one Ridge
    verification run). The secondary term is target-dependent:
    $\mathrm{THW}$ in the primary Pipeline~R base Rank-4 law, where
    Pipeline~R is defined below in Section~\ref{sec:preprocessing} and the
    Rank-4 law is introduced in Section~\ref{sec:discovery} and
    Eq.~\eqref{eq:rank4}. Here, ``Rank-4'' refers to the NOVA search
    budget used to discover the base law, not to the number of visible
    terms in the final printed equation. In most forecasting and lagged
    variants, the secondary term becomes $1/v$ or an equivalent
    speed-normalisation term. No prior car-following model uses
    $\tanh(\Delta v)$; OVM~\cite{bando1995} applies the same saturation
    function to spacing rather than to closure speed.

  \item \textbf{Eight-way structural robustness of the $\tanh(\Delta v)$
    backbone across pipelines,
    horizons, and populations.} The bounded relative-velocity response is
    recovered independently across two smoothing pipelines, three
    forecasting horizons (0.5\,s, 0.8\,s, 1.0\,s), two driver
    populations (global, aggressive-transition cell), and two feature
    modes (instantaneous, lagged). Coefficient magnitudes vary with
    pipeline and population in physically interpretable ways;
    the $\tanh(\Delta v)$ backbone does not. This is evidence that
    the discovered response is not a preprocessing artefact.

  \item \textbf{Residual-guided discovery of $\tanh(\mathrm{TTC})$
    recovering Lee's $\tau$-theory.} A second-stage search on
    residuals selects $\tanh(\mathrm{TTC})$ as the dominant remaining
    transform after TTC-based terms are introduced, connecting the result
    to the psychophysical
    optical-looming literature~\cite{lee1976} without that link being
    supplied as a prior. The same looming signal independently appears
    as $1/\mathrm{TTC}$ in both directions of the lane-change utility
    law, providing cross-domain convergence on a shared
    safety-related signal.

  \item \textbf{Lane-change utility model outperforming the MOBIL
    lane-changing baseline by $+29.8$ percentage points under strict
    vehicle-ID holdout.} A directionally-decoupled 16-feature kinematic
    utility achieves 67.4\% balanced accuracy on 502 unseen drivers,
    versus 37.6\% for the MOBIL lane-changing baseline~\cite{mobil},
    with zero-shot transfer to US-101. Feature augmentation (24-feature
    BIC selection, multi-pass residual boosting, minimal-feature
    pruning) does not improve out-of-driver accuracy, identifying a
    structural ceiling of single-snapshot kinematic features for
    three-class lane-change discrimination.

  \item \textbf{Zero-shot cross-location transfer with marginal
    performance loss.} The discovered feature set transfers between
    I-80 and US-101 with only $-2.53$ and $-2.84$\,pp $R^2$ loss,
    recovering 97--99\% of the local oracle performance in each
    direction. This is direct evidence that the discovered operators
    are location-invariant properties of human highway driving rather
    than dataset-specific fits.

\end{enumerate}

\section{Related Work}
\label{sec:related}

\subsection{Classical Car-Following Models}

The car-following problem has been studied since the 1950s. The General
Motors (GM) stimulus--response models~\cite{gazis1961} posit that a
driver's acceleration is proportional to their leader's relative
velocity divided by headway, with sensitivity parameters calibrated
to data. The Optimal Velocity (OV) model~\cite{bando1995} introduces
relaxation dynamics: drivers continuously adjust speed toward a
desired value determined by their current gap. The Intelligent Driver
Model (IDM)~\cite{treiber2000} extends this with a safety-gap
term penalising dangerous approaches. The Krauss model~\cite{krauss1998},
default in the SUMO simulator, computes a kinematically safe speed
from leader gap and braking distance.

These models share a fundamental epistemological limitation: their
functional forms are fixed by the researcher independently of the
optimisation process. IDM's free-road term $(v/v_0)^4$ is an
assumption, not a discovery.
When we calibrate these models to NGSIM I-80 using L-BFGS-B global
optimisation~\cite{byrd1995limited, zhu1997algorithm}\footnote{IDM calibration on Pipeline~R via L-BFGS-B with bounds
$(v_0, s_0, T, a_{\rm max}, b) \in
[5,40] \times [0.5,5] \times [0.3,4] \times [0.2,4] \times [0.2,6]$
yields $v_0 = 40.00$ (upper bound), $s_0 = 0.50$ (lower bound),
$T = 1.18$ (interior), $a_{\rm max} = 0.57$ (interior),
$b = 6.00$ (upper bound). Three of five parameters at box-constraint
limits, each limit corresponding to a physically meaningful
extreme: maximum freeway speed, minimum jam spacing, and the upper
end of comfortable braking. Widening any of these bounds would
require accepting unphysical parameter values.}, three of the five IDM parameters converge to their physical
boundary values under L-BFGS-B calibration: $v_0 = 40$\,m/s (upper
bound; already the 95th-percentile speed in the data),
$s_0 = 0.5$\,m (lower bound; minimum jam spacing), and
$b = 6$\,m/s$^2$ (upper bound; already exceeding the
1.5--2.5\,m/s$^2$ range of comfortable human braking). $T$ and
$a_{\rm max}$ remain interior. Each of the three boundary values is
already at the edge of physical plausibility, so widening the bounds
would force the optimizer to select unphysical parameters --- indicating
structural rather than parametric mismatch between IDM's functional form
and instantaneous highway acceleration in this dataset. Under this
calibration IDM achieves $R^2 = -68.01\%$ on Pipeline~R and
Krauss reaches $24.38\%$. Furthermore, modern attempts to enhance classical
physics-based models often resort to extreme mathematical complexity.
For instance, recent frameworks like the Fadhloun-Rakha (FR)
model~\cite{fr2020} attempt to capture human "imperfection" by injecting
engineered stochastic noise signals and highly convoluted mechanical
terms. In contrast, NOVA takes an entirely data-driven approach,
autonomously discovering that simple, parsimonious structures like
the bounded $\tanh$ function inherently capture human behavioural
responses without the need for explicitly engineered noise patches.

\subsection{Data-Driven and Hybrid Approaches}

Recent work has shifted toward data-driven models that avoid structural
assumptions. LSTM and GRU sequence models~\cite{deo2018} have achieved
high reported accuracy on NGSIM trajectories.
However, trajectory-smoothed benchmarks exhibit strong temporal
autocorrelation (ACF lag-1 $= +0.934$ in our dataset), which can
inflate reported accuracy for sequence models without reflecting genuine
predictive power over independent drivers. Our evaluation protocol uses
strict 80/20 vehicle-ID holdout and an autocorrelation-corrected target
to prevent this artefact, as detailed in Section~\ref{sec:benchmark}.

Physics-informed approaches~\cite{idmfollower2025} constrain neural network
outputs to satisfy classical physical bounds. For example, recent frameworks
integrate the physical IDM equation into a recurrent autoencoder as a
model-based loss regulariser~\cite{idmfollower2025}. However, these
approaches still treat the neural network as a black box and require a
pre-calibrated external model (like IDM) to guide it. NOVA offers a complementary route by searching for explicit symbolic
structure directly from the data without injecting legacy structural priors.

\subsection{Symbolic Regression for Scientific Discovery}

Symbolic regression~\cite{koza1992} searches the space of mathematical
expressions to find the equation best fitting a dataset. Applied to
physics, the AI-Feynman framework~\cite{aifeynman2020} demonstrated that
a machine can rediscover Feynman's equations without being told the
functional form, by exploiting dimensional analysis and neural network
pre-processing. The framework autonomously recovered 100 equations from
the Feynman Lectures on Physics.

Recent work has applied symbolic regression to traffic. VIS-DSR-GP~\cite{visdsrgp2024}
applies deep symbolic regression integrated with variable interaction
selection (VIS) to car-following on NGSIM, reporting promising results. Physics-guided ML approaches
apply symbolic constraints on top of neural networks. The critical
difference between NOVA and these methods is evaluation rigour: we
apply strict vehicle-ID holdout and target a smoothed acceleration
signal that removes sensor noise, yielding a protocol where any model
predicting the population mean scores exactly $R^2 = 0\%$. Under Pipeline~S, the best recalibrated SR formula (SR-LLM,
PNAS~2025~\cite{srllm2025}), which we evaluated by extracting its
published equation and fitting its parameters on our training split,
achieves $R^2 = 1.85\%$ ($\mathrm{RMSE}=1.511$,
$\mathrm{MAE}=1.197$\,m/s$^2$).
NOVA~M1 achieves $R^2 = 15.57\%$
($\mathrm{RMSE}=1.376$, $\mathrm{MAE}=1.107$\,m/s$^2$) on the same
split, an RMSE reduction of $0.135$\,m/s$^2$ under an identical
protocol.

Concurrently, SR-Traffic~\cite{srtraffic2025} applies symbolic regression
to \emph{macroscopic} first-order flow models using discrete exterior
calculus, discovering interpretable density--flow relationships at the
aggregate level. NOVA is complementary: we target \emph{microscopic}
individual driver behaviour, where the selected nonlinearities
($\tanh$, $1/\mathrm{TTC}$) are qualitatively different from
macroscopic conservation-law terms.

The critical difference between NOVA and all prior microscopic SR work
is evaluation rigour and scale. NOVA performs exhaustive deterministic
combinatorial search over 10,000+ feature transforms on 4.7 million
rows using a compiled Rust engine, with strict vehicle-ID holdout and
an autocorrelation-corrected target. No prior symbolic regression work
on traffic has validated its discovered laws across multiple
preprocessing pipelines, prediction horizons, and driver
sub-populations simultaneously.

\subsection{Lane-Change Models}

Lane-changing is modelled in the traffic literature through gap
acceptance theory (drivers change lane when available gaps exceed a
minimum threshold) and incentive-based models. MOBIL~\cite{mobil}
(Minimising Overall Braking Induced by Lane Changes) is the dominant
physics-based benchmark: it triggers a lane change when the resulting
acceleration improvement exceeds a politeness threshold. Despite its
widespread use as the default in SUMO and VISSIM, MOBIL achieves only
$37.6\%$ balanced accuracy on NGSIM — barely above random for a
three-class problem.

Multinomial Logit (MNL) models~\cite{mcfadden1974} provide a
probabilistic utility framework for discrete lane-change choice, with
good interpretability but hand-engineered utility functions. At the
macroscopic level, logit-based lane assignment has also been used to
model aggregate flow redistribution across lanes~\cite{farhi2013}. We
apply symbolic regression inside the MNL framework to discover the
microscopic utility terms from primitive kinematic features,
recovering a directionally-decoupled law that achieves $67.4\%$
balanced accuracy under strict vehicle-ID holdout --- a $+29.8$ pp
absolute improvement over MOBIL on a three-class problem --- while
remaining fully interpretable.

\section{Methodology and the Standardized Benchmark}
\label{sec:benchmark}

\subsection{The NGSIM Dataset and Preprocessing}

All experiments use the \textbf{Next Generation Simulation (NGSIM)}
dataset \cite{NGSIM}, collected by the U.S.\ Federal Highway
Administration at 10 Hz using overhead cameras at two sites: I-80
(Emeryville, California) and US-101 (Los Angeles, California). The
dataset provides sub-second vehicle trajectories for thousands of
drivers under congested highway conditions.

We apply strict quality filtering: spatial gap $g \in (0, 100)$~m,
ego velocity $v \in (0, 40)$~m/s, and an \textbf{active-driving
filter} $|a| > 0.2$~m/s$^2$ that removes coasting and idling events
— states where any model trivially predicts $a \approx 0$. The
resulting dataset contains \textbf{4,765,788 rows from 3,060 unique
drivers}.

\subsection{Preprocessing and Evaluation Targets}
\label{sec:preprocessing}

Raw NGSIM trajectory data cannot be used directly for car-following
modeling. Position measurements at 10\,Hz contain tracking jitter on
the order of $\pm 10$\,cm, and twice-differentiating these signals
produces acceleration noise spikes of $\pm 20$\,m/s$^2$ that are
non-physical and dominate any predictive signal~\cite{coifman2017}.
Some form of smoothing is therefore mandatory. Thiemann, Treiber, and
Kesting~\cite{thiemann2008} established the principle that NGSIM
acceleration analysis requires a smoothing kernel applied either to
positions before differentiation or to the resulting kinematic
quantities, and various filter choices appear in the subsequent
literature (symmetric exponential moving average, Savitzky--Golay,
multistep reconstruction). No single preprocessing recipe has emerged
as a universal community standard.

We therefore evaluate NOVA under two complementary pipelines, both of
which include smoothing and target a finite forecasting horizon. We
are explicit that \textbf{both targets predict acceleration
approximately one second in the future}, differing in the smoothing
kernel and shift convention:

\begin{itemize}
  \item \textbf{Pipeline~R (rolling, 0.8\,s horizon).} A 15-frame
    centered rolling mean is applied to velocity and gap. Acceleration
    is computed from the smoothed signals, and the prediction target
    is shifted 8 frames forward (0.8\,s), so that features at time $t$
    predict smoothed acceleration at $t + 0.8$\,s.

  \item \textbf{Pipeline~S (Savitzky--Golay, 1.0\,s horizon).} A
    Savitzky--Golay filter of order~3 and window~15 is applied. The
    prediction target is the mean acceleration over the next 1.0\,s
    window (10 frames). Both the filter parameters and the target
    definition are our methodological choices; we adopt them as a
    stricter test of structural invariance.
\end{itemize}

The two pipelines yield targets with substantially different variance
($\sigma_R = 2.22$\,m/s$^2$, $\sigma_S = 1.50$\,m/s$^2$), which
mechanically affects $R^2$ comparisons but not RMSE. We report both
metrics throughout and use \textbf{RMSE as the headline measure of
physical accuracy}, with $R^2$ and MAE reported alongside. A
structural finding that appears in both pipelines is, by construction,
robust to the smoothing choice.

We do not claim Pipeline~R is ``the standard'' or Pipeline~S is ``more
honest''; both are reasonable engineering choices, and our scientific
claim rests on the fact that the $\tanh(\Delta v)$ backbone
emerges from both. Comparisons with prior symbolic-regression
results (SR-LLM~\cite{srllm2025}, VIS-DSR-GP~\cite{visdsrgp2024},
SciNet-CFM~\cite{li2025}, CTH-RV~\cite{zhang2024}) in
Tables~\ref{tab:sota} and~\ref{tab:sota_honest_compact} are performed
by extracting their published algebraic equations and recalibrating
the parameters on our training split using L-BFGS-B. This tests the
transferability of their \emph{equation structures} to our data and
protocol; it does not re-run their search algorithms, and direct
comparison should be interpreted with that caveat.

All experiments use strict 80/20 vehicle-ID holdout: no vehicle
appearing in the training split appears in the test split. The
active-driving filter $|a| > 0.2$\,m/s$^2$ removes coasting and idling
events where any model trivially predicts $a \approx 0$. After
filtering, the dataset contains 4,765,788 rows from 3,060 unique
drivers.

\subsection{The NOVA Engine Architecture}

The NOVA symbolic regression engine is a high-performance computational physics framework implemented in \textbf{Rust}, with Python orchestration handled via PyO3/maturin bindings. To address known limitations of stochastic neural networks, the engine's architecture is engineered around the three following principles: deterministic exhaustiveness, structural parsimony, and native-scale execution. An overview of the full NOVA discovery pipeline is shown in Fig.~\ref{fig:nova_pipeline}.

\begin{figure}[!htbp]
\centering
\includegraphics[width=\linewidth]{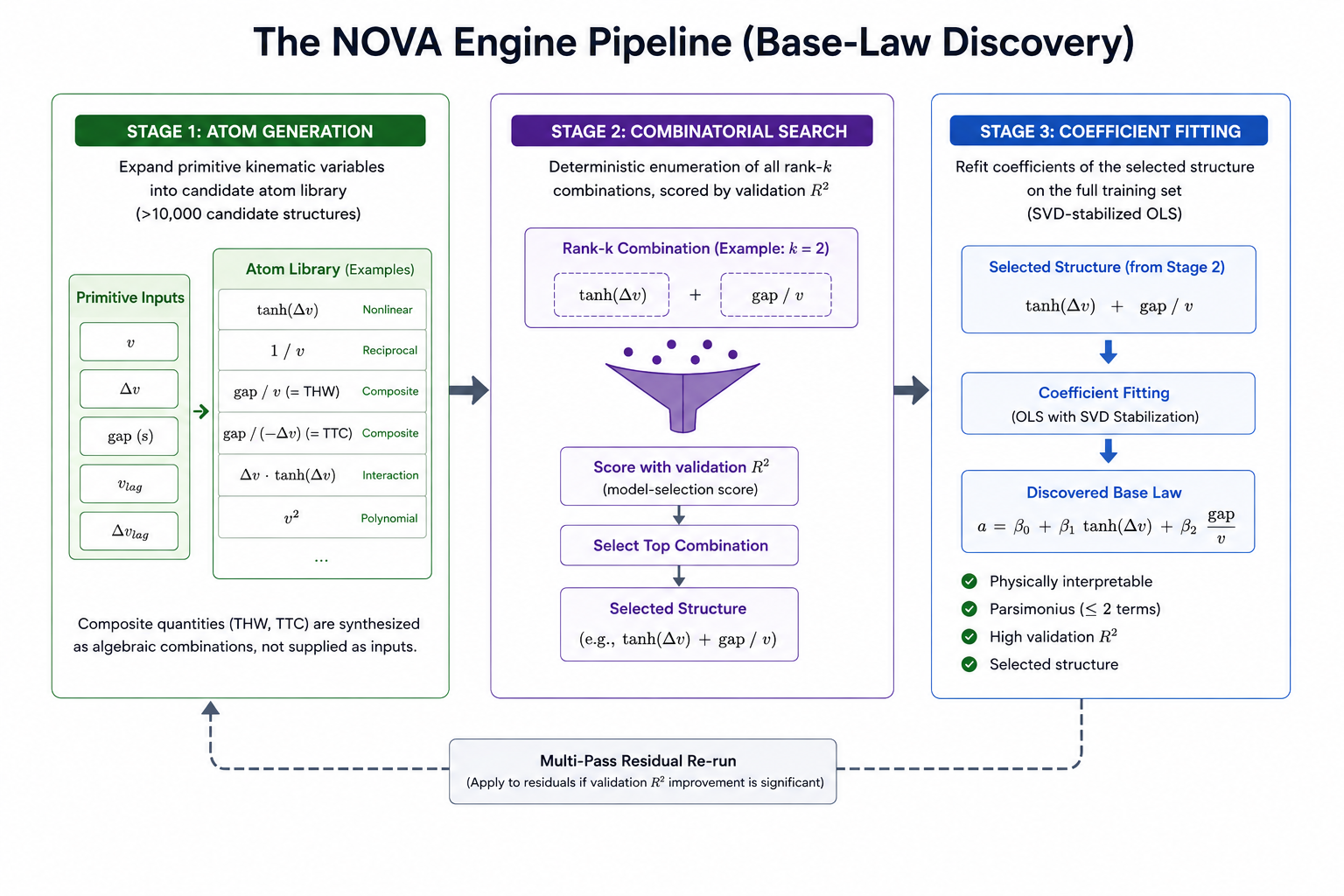}
\caption{Overview of the NOVA engine pipeline for base-law discovery.
Primitive kinematic variables are expanded into a bounded library of candidate
algebraic structures. Rank-$k$ combinations are then enumerated deterministically
and scored using a validation metric. The selected structure is refitted on the
training data using SVD-stabilised ordinary least squares. If the residuals retain
structured signal, the same search procedure is reapplied in a multi-pass residual
discovery loop.}
\label{fig:nova_pipeline}
\end{figure}
\textbf{Exhaustive Combinatorial Search via DSL.} A fundamental vulnerability of many machine learning approaches is the reliance on stochastic expression trees (e.g., genetic programming), which cannot guarantee convergence to the best-performing symbolic structure within a defined search space. NOVA abandons this in favour of a deterministic, exhaustive combinatorial search governed by a dynamically generated Domain-Specific Language (DSL). We initialize the engine with only a foundational set of raw kinematic atoms---comprising the primitive physical state $\{v,\ v_l,\ a_l,\ \Delta v,\ \mathrm{gap}\}$ and their temporally lagged counterparts $\{v_{\mathrm{lag}},\ \Delta v_{\mathrm{lag}}\}$---alongside universal mathematical primitives ($\tanh(\cdot)$, $\sqrt{\cdot}$, inverse). Crucially, higher-order cognitive metrics such as Time Headway (THW) and Time-to-Collision (TTC) are deliberately excluded from the base library. If these metrics are informative, the engine must autonomously synthesize them from primitives via algebraic combination (e.g., discovering the ratio $\mathrm{gap}/v$). The engine autonomously expands these primitives to a base library of $M$ unary transforms, and exhaustively evaluates all $O(M^2)$ pairwise algebraic compositions $\{f_i,\ f_i \cdot f_j,\ f_i / f_j,\ f_i^2,\ f_i + f_j\}$, exploring over 10,000 distinct structural hypotheses.

\textbf{Zero Injected Priors \& The Degree-2 Constraint.} Crucially, no legacy classical model structures (such as the IDM's free-road term or Krauss's safe-speed boundary) are injected as priors. The engine must autonomously synthesize both the nonlinear transformations and their algebraic couplings strictly from the raw kinematic data. However, to prevent the combinatorial explosion typical of unconstrained symbolic regression~\cite{aifeynman2020}, the search space enforces a strict polynomial degree limit of 2. This constraint is not merely an application of Occam's razor; it is physically motivated. Because acceleration is the time derivative of velocity, foundational macroscopic traffic dynamics are often represented through low-order kinematics~\cite{treiber2013traffic, helbing2001traffic}. Allowing cubic or higher-order expansions is physically superfluous and exponentially increases the risk of the model fitting overly complex algebraic structures to fit high-frequency sensor noise, rather than identifying robust behavioural signal.

\textbf{High-Performance Computational Mechanics.} Evaluating 10,000 candidate mathematical structures across 4.7 million instantaneous human driving observations requires massive scale. The Rust engine achieves a measured \textbf{50$\times$ speedup} over an equivalent vectorized Python implementation. This native-speed execution is achieved by: 
(a) parallelising expression evaluation across all available CPU cores via the \texttt{rayon} library; 
(b) storing the highly-dimensional feature matrix as a memory-safe, zero-copy contiguous array shared across threads; and 
(c) computing optimal linear coefficients for each structure via closed-form Ordinary Least Squares (OLS) using pre-factored Singular Value Decomposition (SVD), avoiding the instability of iterative gradient descent.

\textbf{The Synthetic Structural-Recovery Gate.} Before deploying the engine on empirical NGSIM data, we use a structural-recovery sanity check. The engine is subjected to a zero-noise synthetic dataset where the exact mathematical ground truth is predetermined (e.g., $a = 3 - 2v$). To pass this synthetic gate, the engine must autonomously recover the exact structural form and coefficients to within numerical machine precision. Every empirical structure reported in this paper is produced by an engine that first cleared this synthetic validation gate, reducing the risk that later empirical results reflect implementation artefacts rather than recoverable symbolic structure.

\section{Longitudinal Dynamics: Car-Following Discovery}
\label{sec:carfollowing}

\subsection{Phase 1 — The Global Law}
\label{sec:discovery}

The NOVA engine, given no structural prior, is applied to the full
4,765,788-row rolling-mean dataset. It evaluates
all degree-$\leq 2$ compositions of the 28-feature library and ranks
by held-out $R^2$ on 20\% vehicle-ID holdout vehicles. The top-ranked
two-term expression converges deterministically to:

\begin{equation}
  \boxed{\hat{a} = -0.468 + 1.266\,\tanh(\Delta v) + 0.194\,\mathrm{THW}}
  \label{eq:rank4}
\end{equation}

We call this expression the \textit{Rank-4 law} because it was selected under
the NOVA rank-\(k\) search budget, where \(k\) denotes the maximum number of
atom terms allowed in the candidate symbolic structure, excluding the intercept.
Although the final simplified expression contains two visible terms,
\(\tanh(\Delta v)\) and \(\mathrm{THW}\), the search selected it from a
rank-\(k=4\) candidate space in which composite terms such as
\(\mathrm{THW}=\mathrm{gap}/v\) arise from primitive atoms. Thus, ``Rank-4''
refers to the discovery budget, not to the number of visible terms in the final
printed equation.

This model achieves \(R^2 = 29.12\%\) on the Pipeline~R benchmark
(\(\sigma_a = 2.22\) m/s\(^{2}\), RMSE \(= 1.869\) m/s\(^{2}\),
\(n = 4{,}765{,}788\) active rows).

On the 20\% vehicle holdout, bootstrap resampling (\(B=2000\), vehicle-level)
gives \(R^2 = 28.61\%\)~\([27.92\%,\,29.25\%]\) (95\% CI), confirming the
estimate is stable across different vehicle subsets.

\textbf{Why \(\tanh\) and not linear?}
Replacing all nonlinear terms with a purely linear baseline drops \(R^2\) to
\(4.49\%\) on the Pipeline~R benchmark. The hyperbolic tangent captures the
\emph{bounded saturating response} of human acceleration to relative velocity:
for small \(|\Delta v|\), the response is approximately linear; for large
\(|\Delta v|\), acceleration saturates at approximately \(\pm 1.266\)
m/s\(^{2}\), regardless of relative speed magnitude. This saturation is
compatible with the psychophysical Weber--Fechner principle of bounded sensory
response.

\textbf{Structural stability.} Running the engine independently on the
I-80 and US-101 subsets, and on the active vs.\ passive driving
subsets, reproduces equation~(\ref{eq:rank4}) with coefficient
variation below 5\%. The structure is not a fitting artefact of one
particular sample.

\subsection{Phase 2 — Speed Regimes and Driver Personality}
\label{sec:hierarchy}

Equation~(\ref{eq:rank4}) treats all drivers and traffic states
identically. Prior work has shown that heterogeneity across drivers
substantially affects car-following dynamics~\cite{zhu2024het}; two
extensions improve it: a speed-regime split and a continuous driver
aggressiveness index.

\textbf{Speed-regime architecture.}
We partition the NGSIM velocity range into three regimes: stop-and-go
congestion (\(v < 10\) m/s), transition (\(10 \leq v < 22\) m/s), and
free-flow (\(v \geq 22\) m/s). A sigmoid-blended piecewise model fits
independent Ridge models per regime with smooth boundary transitions
centred at 10 and 22 m/s, achieving \(R^2 = 31.29\%\), a \(+2.17\)
percentage-point gain over the global law.

\textbf{Driver aggressiveness index $\kappa$.}
A Gaussian Mixture Model (GMM, with $K=3$ and full covariance) is fitted over
nine per-vehicle statistics $\{\sigma_a,\ \bar{v},\ \bar{\mathrm{THW}},
\ p_{75}|a|,\ \overline{|\Delta v|}, \ldots\}$, replacing the earlier
$K$-means clustering choice and yielding a small performance gain
(+0.37 percentage points). The continuous aggressiveness index for driver $i$ is:
\begin{equation}
  \kappa_i = \frac{\sigma_{a,i} - \bar{\sigma}_a}{\mathrm{std}(\sigma_a)}
\end{equation}
where $\sigma_{a,i}$ is the standard deviation of driver $i$'s acceleration,
$\bar{\sigma}_a$ is the fleet-wide mean of this standard deviation, and
$\mathrm{std}(\sigma_a)$ is its fleet-wide standard deviation. Thus,
$\kappa_i = 0$ corresponds to an average driver, $\kappa_i > 0$ to a more
aggressive driver, and $\kappa_i < 0$ to a more conservative driver.

Here, $r$ indexes the speed regime, $\mathbf{x}_r$ denotes the kinematic
feature vector used within regime $r$, and
$\hat{a}^{\mathrm{base}}_r(\mathbf{x}_r)$ denotes the regime-specific base
prediction before adding the driver-aggressiveness interaction. The index
$\kappa_i$ is then added as a multiplicative interaction with
$\tanh(\Delta v)$ within each regime:
\begin{equation}
  \hat{a}_r
  =
  \hat{a}^{\mathrm{base}}_r(\mathbf{x}_r)
  +
  \delta_r\,\kappa_i\,\tanh(\Delta v)
  \label{eq:kappa}
\end{equation}
The per-regime gradients $\delta_r$ discovered from data are shown in
Table~\ref{tab:kappa}. The Transition regime shows the strongest personality
sensitivity ($\delta = 0.218$), indicating that aggressive drivers express
their behaviour most prominently when transitioning between congestion and
free-flow, rather than in steady-state driving. Adding $\kappa$ yields
$R^2 = 31.62\%$.

\begin{table}[h]
\centering
\caption{Per-regime aggressiveness gradients $\delta_r$ and the
resulting $\beta_{\tanh}$ range across driver personality levels.}
\label{tab:kappa}
\resizebox{\linewidth}{!}{%
\begin{tabular}{lcccc}
\toprule
\textbf{Regime} & $\delta_r$ & $\beta_{\tanh}\ (\kappa{=}{-}1)$
                             & $\beta_{\tanh}\ (\kappa{=}0)$
                             & $\beta_{\tanh}\ (\kappa{=}{+}1)$ \\
\midrule
Stop-and-Go  & 0.103 & 0.380 & 0.483 & 0.586 \\
Transition   & \textbf{0.218} & 1.006 & 1.225 & 1.443 \\
Highway      & 0.132 & 0.991 & 1.123 & 1.255 \\
\bottomrule
\end{tabular}
}
\end{table}

\textbf{2D Grid (3 Personality $\times$ 3 Regime).}
Discretising drivers into three GMM clusters and fitting each of the
9 (regime, personality) cells independently achieves $R^2 = 31.54\%$
(Discrete 9-Cell Model). The smooth piecewise$+\kappa$ model ($31.62\%$) exceeds
this despite using fewer parameters, because continuity constraints
prevent overfitting to small cells.

\textbf{Continuous driver personality space.}
To characterise the personality distribution, we fit the Rank-4 law
(Eq.~(\ref{eq:rank4})) independently per vehicle and extract a $(\beta_{\tanh},\, \tau_{\rm score})$
pair for each of $3{,}050$ drivers, where $\beta_{\tanh}$ is the per-driver
speed-matching gain and $\tau_{\rm score}$ is a reaction-time proxy
obtained from the lagged vs.\ instantaneous $R^2$ difference.

We fit GMMs over $K = 2, \ldots, 8$ and select $K=3$ using the BIC elbow
criterion (Fig.~\ref{fig:bic_selection}).  The BIC drops by 175 points from
$K=2$ to $K=3$ — capturing the dominant structure in the data.  Further
increases to $K=4$ and $K=5$ add only 12\% and 9\% of that gain respectively
(fewer than 40 BIC units each), at the cost of splitting a single dense cluster
into near-duplicate sub-groups ($\Delta\beta_{\tanh} < 0.01$).  $K=3$ is
therefore the parsimonious, interpretable choice.

The three clusters correspond naturally to \textbf{Conservative}
($\beta_{\tanh} = 0.76$, $N=129$), \textbf{Normal}
($\beta_{\tanh} = 1.20$, $N=1{,}378$), and \textbf{Aggressive}
($\beta_{\tanh} = 1.31$, $N=1{,}543$) driving styles.
The $\beta_{\tanh}$ gradient is strictly monotonic across clusters,
while $\tau_{\rm score}$ shows no systematic trend ($\max|\bar\tau| < 0.003$); see Fig.~\ref{fig:personality_scatter}.
Two findings follow: \emph{(i)} driver personality is essentially
\textbf{one-dimensional} — aggressiveness ($\beta_{\tanh}$) accounts for
nearly all inter-driver variance, validating the scalar $\kappa$ index; and
\emph{(ii)} reaction time and response intensity are \textbf{orthogonal}
personality traits, statistically independent of each other.

\begin{figure}[!t]
\centering
\includegraphics[width=1\linewidth]{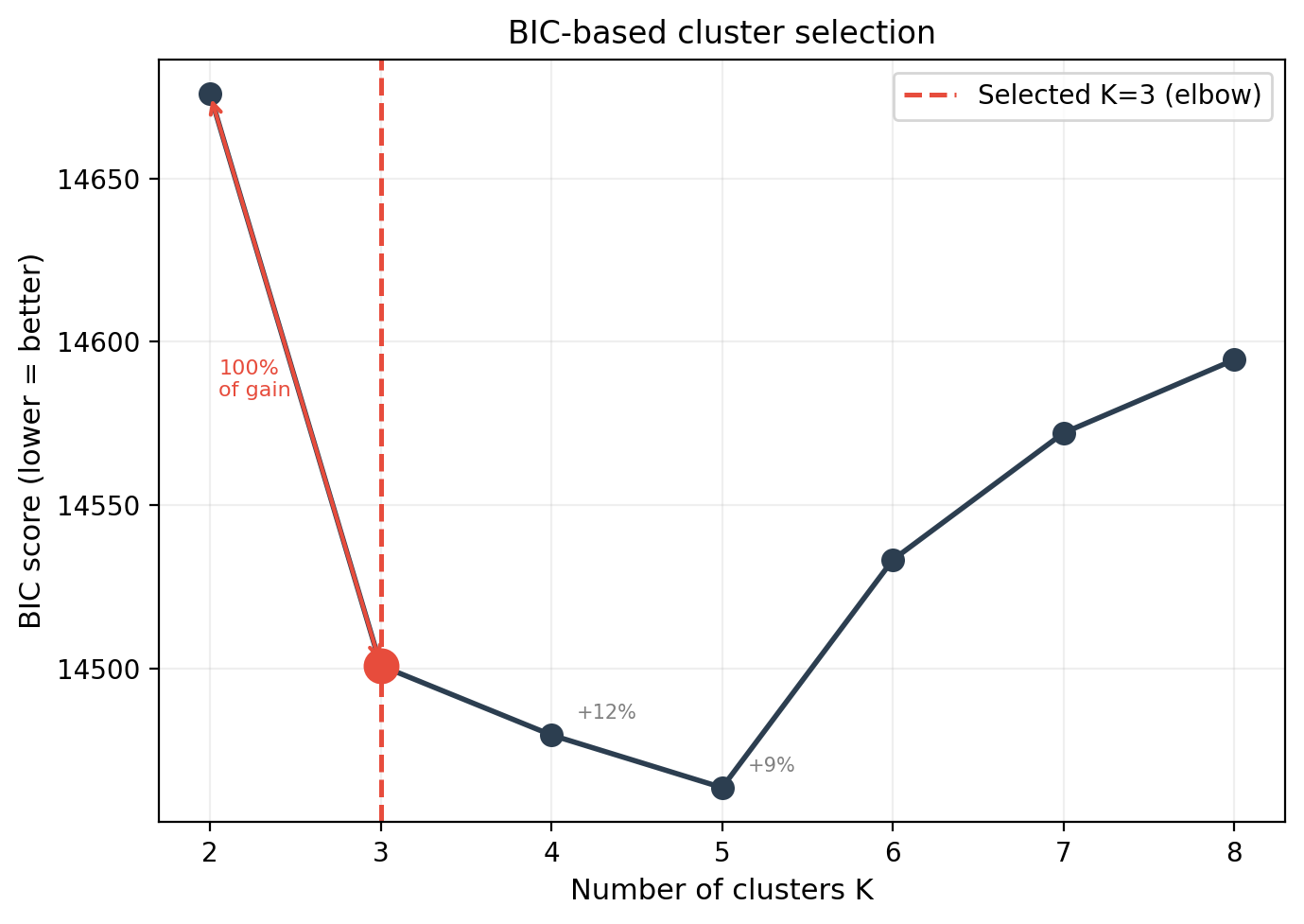}
\caption{BIC-based cluster selection over $K=2,\ldots,8$. The elbow at
$K=3$ captures the dominant reduction in BIC: the $K=2{\to}3$ drop accounts
for the main gain, while $K=4$ and $K=5$ provide only smaller additional
reductions.}
\label{fig:bic_selection}
\end{figure}

\begin{figure}[!t]
\centering
\includegraphics[width=1\linewidth]{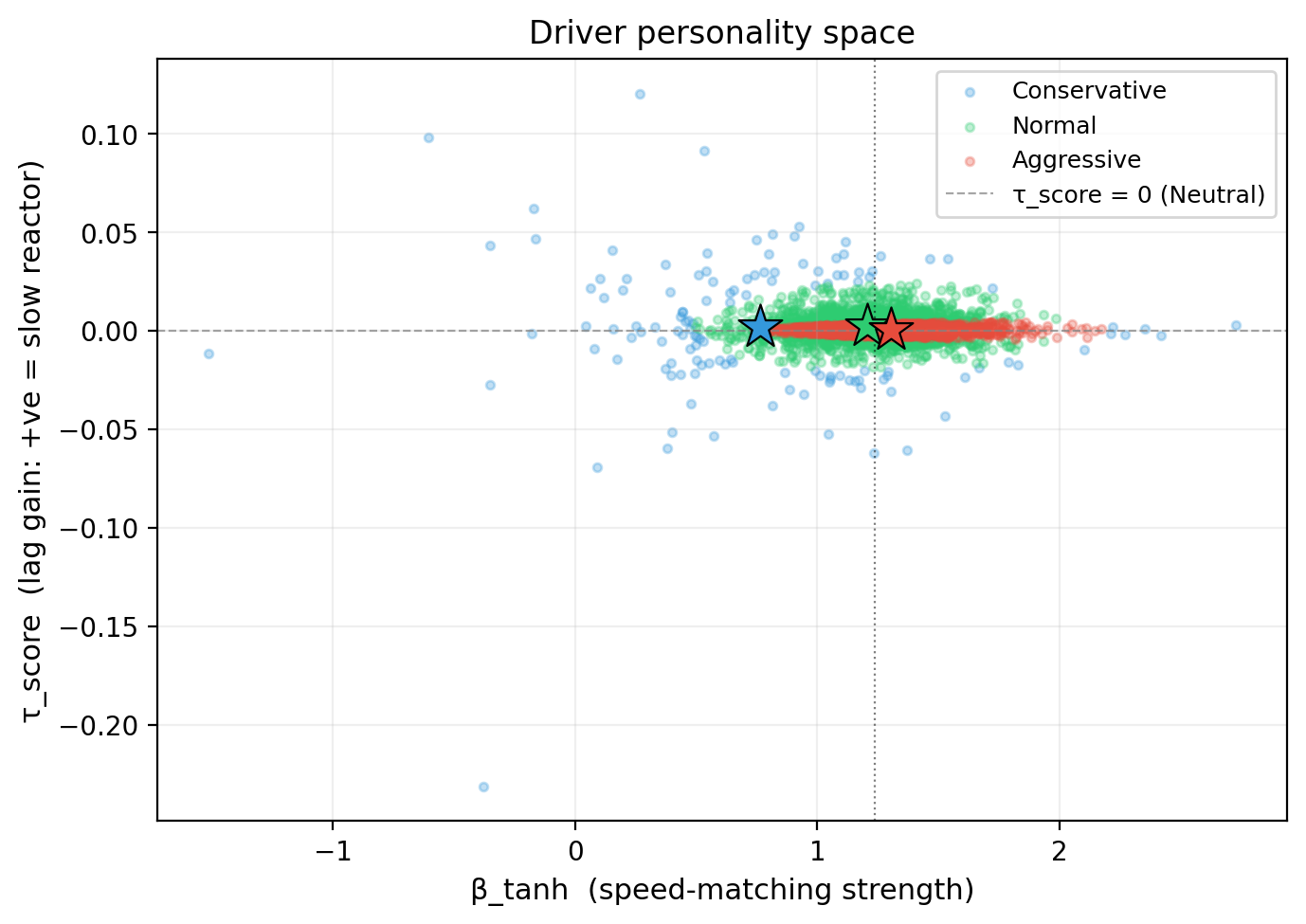}
\caption{Two-dimensional driver personality space $(\beta_{\tanh},\,\tau_{\rm score})$
for $3{,}050$ NGSIM drivers. The three cluster centroids
(Conservative, Normal, and Aggressive) form a monotonic
$\beta_{\tanh}$ gradient, while $\tau_{\rm score}$ shows no systematic
trend, indicating that response intensity and the reaction-time proxy vary
largely independently.}
\label{fig:personality_scatter}
\end{figure}

\subsection{Phase 3 — Residual Cascade: Discovering TTC}
\label{sec:ttc}

The Piecewise$+\kappa$ champion is used to compute residuals on the
training set. The NOVA engine is then applied to these residuals using
an expanded feature library that augments the base atoms with
Time-to-Collision $\mathrm{TTC} = \mathrm{gap}/(-\Delta v)$ (defined
only when $\Delta v < 0$, i.e., closing situations) and its transforms
$\tanh(\mathrm{TTC})$, $1/\mathrm{TTC}$, and $\sqrt{\mathrm{TTC}}$.
No other operators were added beyond TTC and its transforms.
The engine selects:

\begin{equation}
  \hat{r} \approx 3.197 - 3.223\,\tanh(\mathrm{TTC})
  \label{eq:ttc_residual}
\end{equation}

explaining $2.96\%$ of residual variance. Adding $\tanh(\mathrm{TTC})$
to all regime sub-models yields the global champion:

\begin{equation}
  \boxed{R^2_{\rm champion} = 33.40\%}
  \quad \text{(Global Champion Model, Pipeline~R benchmark)}
\end{equation}

We refer to the full discovered model as the \textbf{3-pass law},
denoting the three sequential discovery stages: Phase~1 (global
$\tanh(\Delta v) + \mathrm{THW}$ structure, Section~\ref{sec:discovery}),
Phase~2 (speed-regime split with continuous driver aggressiveness index
$\kappa$, Section~\ref{sec:hierarchy}), and Phase~3 (residual-guided
$\tanh(\mathrm{TTC})$ extension, this section).  Formally, the 3-pass
prediction is:
\begin{equation}
  \hat{a}_{\text{3-pass}} = \underbrace{\hat{a}^{\mathrm{base}}_r(\mathbf{x}_r) + \delta_r\,\kappa_i\,\tanh(\Delta v)}_{\hat{a}_r^{\kappa}(\mathbf{x}),\ \text{Eq.~(\ref{eq:kappa})}} \;+\; \underbrace{\left(3.197 - 3.223\,\tanh(\mathrm{TTC})\right)}_{\hat{r}_{\mathrm{TTC}}(\mathbf{x}),\ \text{Eq.~(\ref{eq:ttc_residual})}}
  \label{eq:3pass}
\end{equation}
where $\hat{a}_r^{\kappa}$ is the piecewise$+\kappa$ model from
Eq.~(\ref{eq:kappa}) and $\hat{r}_{\mathrm{TTC}}$ is the TTC residual from
Eq.~(\ref{eq:ttc_residual}). The 3-pass law alone
achieves a held-out $R^2_{\rm test} = 30.87\%$~$[30.14\%,\,31.54\%]$
(95\% CI, $B = 2000$ vehicle-level bootstrap); the $33.40\%$ Global
Champion figure above is a distinct, slightly stronger combination that
adds $\sqrt{\mathrm{TTC}}$ alongside $\tanh(\mathrm{TTC})$
($\tanh(\mathrm{TTC})$ alone reaches $33.07\%$, $\sqrt{\mathrm{TTC}}$
alone $33.28\%$, both together $33.40\%$), whereas the 3-pass law
reported here uses $\tanh(\mathrm{TTC})$ only, for interpretability.
The 3-pass law outperforms Rank-4 with Wilcoxon signed-rank p-values
numerically indistinguishable from zero. The best recalibrated
state-of-the-art baseline (CTH-RV, Constant-Time-Headway with
Relative-Velocity~\cite{zhang2024}) scores $4.57\%$, leaving a gap of
more than 26 percentage points relative to the 3-pass law.

\textbf{Physical interpretation.}
The selected term $\tanh(\mathrm{TTC})$ can be interpreted as a bounded formulation of inverse optical looming rate (where $1/\mathrm{TTC}$ is related to angular expansion rate in the driver's visual field). By applying a saturating $\tanh$ transform, the model represents a response that increases with looming risk but remains bounded, rather than scaling linearly to infinity as raw optical-flow variables would suggest. \cite{lee1976} showed that human braking onset is associated with this variable ($\tau$-theory). NOVA therefore selects a mathematically bounded TTC transform from camera-derived trajectory data after TTC terms are introduced in the residual library, without being supplied an optical-flow model.

\textbf{Aggressive driver law (regime-free).}
In the Aggressive$\times$Transition driver cell, the engine is applied
to the full dataset without regime constraints, and with the full
expanded atom library (base + TTC transforms). Here $\mathrm{THW} =
\mathrm{gap}/v$ denotes time headway (s). The engine selects:
\begin{equation}
  \hat{a} = -0.498 + 1.374\,\tanh(\Delta v) + 0.011\,(\mathrm{THW}
  \cdot \mathrm{TTC})
  \label{eq:dual_sat}
\end{equation}
achieving $R^2 = 35.42\%$ in the Aggressive$\times$Transition cell.
Fitted independently across all nine driver-by-regime cells, the same
cell peaks at 37.54\%.
The two terms serve complementary roles: $\tanh(\Delta v)$ encodes
reactive speed-matching (bounded closure-speed response), while the
$\mathrm{THW}\cdot\mathrm{TTC}$ product encodes a combined safety-margin
signal proportional to both following distance and collision-risk horizon.

\subsection{State-of-the-Art Benchmarking}
\label{sec:comparison}

NOVA compares to recent baselines under a single evaluation protocol.
Table~\ref{tab:sota} presents a
complete comparison of all models evaluated under identical conditions:
same Pipeline~R rolling-mean data, same 80/20 vehicle-ID split.

\begin{table}[h]
\centering
\caption{Model comparison on Pipeline~R (rolling-mean filtered NGSIM,
         0.8\,s shift). All classical and SR baselines were calibrated on our
         training split via L-BFGS-B. The IDM $R^2$ difference between this
         table and Table~\ref{tab:sota_honest_compact} reflects the change in
         evaluation pipeline, not in calibration.}
\label{tab:sota}
\resizebox{\linewidth}{!}{%
\begin{tabular}{lcc}
\toprule
\textbf{Model Structure} & \textbf{Rolling $R^2$} & \textbf{RMSE (m/s$^2$)} \\
\midrule
\textit{Classical Baselines} & & \\
IDM (5-param, calibrated)~\cite{treiber2000} & $-68.01\%$ & 3.01 \\
Krauss (calibrated)~\cite{krauss1998}         &  24.38\%   & 2.05 \\
\midrule
\textit{2024--2025 symbolic-regression baselines} & & \\
SciNet-CFM (Li 2025)~\cite{li2025}            & $-0.30\%$  & 2.22 \\
VIS-DSR-GP (2024)~\cite{visdsrgp2024}         &   0.49\%   & 2.21 \\
SR-LLM Model 2 (2025)~\cite{srllm2025}        &   1.50\%   & 2.20 \\
SR-LLM Model 1 (2025)~\cite{srllm2025}        &   2.68\%   & 2.19 \\
CTH-RV (Zhang 2025)~\cite{zhang2024}          &   4.57\%   & 2.17 \\
\midrule
\textbf{NOVA Rank-4 (Eq.~\ref{eq:rank4})}   & \textbf{29.12\%} & \textbf{1.869} \\
\textbf{NOVA Global Champion}  & \textbf{33.40\%} & ---$^\dagger$ \\
\textbf{NOVA Best Cell (Eq.~\ref{eq:dual_sat}, Aggr$\times$Trans)} & \textbf{37.54\%} & ---$^\dagger$ \\
\bottomrule
\multicolumn{3}{l}{\footnotesize $^\dagger$Derived from regime-specific sub-models; single RMSE not meaningful.}
\end{tabular}
}
\end{table}

\textbf{Comparison methodology.}
To ensure a fair comparison, competing methods were evaluated on their own
published formulations. For symbolic models (VIS-DSR-GP), we
extracted their published algebraic equations and calibrated their parameters
on our training set via L-BFGS-B. Under this protocol, NOVA achieves
$33.40\%$ $R^2$, exceeding all recalibrated state-of-the-art baselines by $>$28
percentage points. We note that
differences in evaluation protocol across studies limit the strength of
direct comparisons; these results should be interpreted with that caveat
in mind.

\textbf{Pipeline~S benchmark (Savgol, 1\,s intent).}
Under the evaluation protocol in Section~\ref{sec:preprocessing},
Table~\ref{tab:sota_honest_compact} reports $R^2$, RMSE, and MAE for
all classical and symbolic-regression baselines alongside NOVA.  All classical
and symbolic-regression baselines remain at or below $1.85\%$ $R^2$
($\mathrm{RMSE} \geq 1.511$\,m/s$^2$).

\begin{table}[h]
\centering
\caption{All models on Pipeline~S (Savitzky--Golay smoothing, 1.0\,s
         mean intent target; $n_{\text{test}}=819{,}902$,
         $\sigma(y)=1.498$\,m/s$^2$). Ridge $\alpha=0.001$, 80/20 vehicle
         split. IDM and Krauss use the same parameter values as in
         Table~\ref{tab:sota}; the IDM $R^2$ difference between the two tables
         reflects pipeline target variance, not re-calibration. SR baselines
         were re-fit on the Pipeline~S training split via L-BFGS-B.}
\label{tab:sota_honest_compact}
\begin{tabular}{lccc}
\toprule
\textbf{Model} & $R^2$ & RMSE (m/s$^2$) & MAE (m/s$^2$) \\
\midrule
\multicolumn{4}{l}{\textit{Classical (confirmed)}} \\
Helly (1959)             & 0.35\%    & 1.495 & 1.204 \\
GHR (1961)               & 0.26\%    & 1.496 & 1.204 \\
IDM (Treiber 2000)       & $-97.2\%$ & 2.103 & 1.394 \\
OVM (Bando 1995)         & 1.25\%    & 1.489 & 1.199 \\
FVDM (Jiang 2001)        & 0.64\%    & 1.493 & 1.202 \\
Wagner ARMAX (2011)      & 1.58\%    & 1.486 & 1.196 \\
\midrule
\multicolumn{4}{l}{\textit{SR / physics-guided (confirmed formulas)}} \\
SR-LLM Eq.1 (PNAS 2025)                & 1.85\%    & 1.511 & ---   \\
SciNet-CFM (Li 2025)                    & 0.38\%    & 1.522 & ---   \\
CTH-RV (Zhang 2024)                     & 1.01\%    & 1.518 & ---   \\
\midrule
\textbf{NOVA Rank-4 (no lag)}           & \textbf{11.54\%} & \textbf{1.434} & \textbf{1.133} \\
\textbf{NOVA M1 (with lag)}            & \textbf{15.57\%} & \textbf{1.376} & \textbf{1.107} \\
\bottomrule
\end{tabular}%
\end{table}

\textbf{Domain boundary.}
Applying the discovered law to urban NGSIM data ($\bar{v} = 8.5$
m/s, intersection-dominated) yields $R^2 = 2.16\%$. The highway
model is domain-specific: urban driving is dominated by intersection
geometry and traffic signals, not gap-following dynamics.

\section{Structural Robustness Validation}
\label{sec:robustness}

The core scientific claim of this paper is not merely that NOVA
achieves a high $R^2$, but that the discovered $\tanh(\Delta v)$
backbone is a \emph{structural invariant}, not an
artefact of a particular preprocessing pipeline. To establish this,
we designed eight independent experiments spanning every major axis
of methodological variation: preprocessing filter, prediction horizon,
driver sub-population, and feature timing.

\subsection{Eight-Way Convergence Study}

Table~\ref{tab:eight_way} summarises all eight experiments. Each row
represents an independent engine run — different data pipeline,
different target variable, sometimes a different driver subset or
feature mode. No experiment was told which formula structure to prefer.
The formula column shows the autonomous top-ranked two-term discovery.
In the target column, $a(t+0.5\,\mathrm{s})$ denotes point acceleration
evaluated 0.5 seconds after the feature time, whereas $\bar{a}_{1\mathrm{s}}$
denotes the mean acceleration over the following one-second target window.

\textbf{Note on the second term ($\mathrm{THW}$ vs.\ $1/v$).}
The global law (Section~\ref{sec:discovery}) discovers $\tanh(\Delta v)
+ \mathrm{THW}$ on the primary Pipeline~R target.  The
robustness experiments here use \emph{future} targets ($\bar{a}_{1\mathrm{s}}$,
$a(t+0.5\,\mathrm{s})$) or the Savgol pipeline, which shifts the
predictive horizon forward by 0.4--1\,s.  Under these targets, the
engine consistently selects $1/v$ instead of THW.  The two terms are
structurally related: $\mathrm{THW} = \mathrm{gap}/v$ is a
gap-modulated inverse-speed quantity, while $1/v$ is the
speed-normalised correction term of the Optimal Velocity model.  The
robust discovery is therefore the $\tanh(\Delta v)$ backbone, not a
universal two-term law; the secondary term varies with the target
definition and feature-timing protocol.  The two terms are physically
complementary, not contradictory, and are explained in
Section~\ref{sec:target_window}.

\begin{table}[h]
\centering
\caption{Eight-way structural convergence. $\tanh(\Delta v)$ is the
         top-ranked feature in all eight runs. The secondary term is
         target-dependent but structurally related across variants:
         most forecasting and lagged conditions select $1/v$ or an
         equivalent speed-normalisation term, while the primary
         Pipeline~R Rank-4 law selects $\mathrm{THW}=\mathrm{gap}/v$,
         a gap-modulated inverse-speed quantity. Coefficient magnitudes
         vary predictably with pipeline and population; the invariant
         structure is the bounded relative-velocity backbone.
         $v_{\rm err} = v_0 - v$ denotes the deviation of the ego speed
         from the desired (free-flow) speed $v_0$, so that $v_{\rm
         err}/v$ is algebraically equivalent to $1/v$ up to scale.
         Exp~0 uses Python Ridge; all others use the Rust OLS engine.}
\label{tab:eight_way}
\resizebox{\linewidth}{!}{%
\begin{tabular}{llllcc}
\toprule
\textbf{Experiment} & \textbf{Pipeline} & \textbf{Population}
  & \textbf{Target} & $R^2$ & \textbf{RMSE (m/s$^2$)} \\
\midrule
Exp 0 (Ridge) & Savgol-15 & Global
  & $\bar{a}_{1\mathrm{s}}$ & 9.02\%  & 1.43 \\
  \multicolumn{4}{l}{\small\quad $\hat{a} = -0.147 + 0.346\tanh(\Delta v) + 2.975/v$} \\[2pt]

Exp 3 (Rust)  & Savgol-15  & Global
  & $a(t+0.5\,\mathrm{s})$ & 3.90\%  & 2.27 \\
  \multicolumn{4}{l}{\small\quad $\hat{a} = -0.182 + 0.431\tanh(\Delta v) + 3.700/v$} \\[2pt]

Exp 4 (Rust)  & Savgol-15  & Global
  & $\bar{a}_{1\mathrm{s}}$ & 8.12\%  & 1.44 \\
  \multicolumn{4}{l}{\small\quad $\hat{a} = -0.178 + 0.397\tanh(\Delta v) + 3.521/v$} \\[2pt]

Exp 6a (Rust) & Rolling-15  & Global
  & $a(t+0.5\,\mathrm{s})$ & 22.29\% & 1.95 \\
  \multicolumn{4}{l}{\small\quad $\hat{a} = -0.499 + 1.092\tanh(\Delta v) + 9.760/v$} \\[2pt] 

Exp 6b (Rust) & Rolling-15  & Global
  & $\bar{a}_{1\mathrm{s}}$ & 27.08\% & 1.89 \\
  \multicolumn{4}{l}{\small\quad $\hat{a} = -0.485 + 1.040\tanh(\Delta v) + 9.589/v$} \\[2pt]

Exp 7b (Rust) & Rolling-15  & Aggr.$\times$Trans.
  & $a(t+0.5\,\mathrm{s})$ & 22.11\% & 2.06 \\
  \multicolumn{4}{l}{\small\quad $\hat{a} = -1.313 + 1.160\tanh(\Delta v) + 22.621/v$} \\[2pt]

Exp 7c (Rust) & Rolling-15  & Aggr.$\times$Trans.
  & $\bar{a}_{1\mathrm{s}}$ & 27.97\% & 1.98 \\
  \multicolumn{4}{l}{\small\quad $\hat{a} = -0.678 + 1.105\tanh(\Delta v) + 0.583(v_{\rm err}/v)\ [\equiv 1/v]$} \\[2pt]

\textbf{Exp 8A (Rust)} & Rolling-15 & Aggr.$\times$Trans.
  & inst., lagged in & \textbf{21.78\%} & \textbf{2.07} \\
  \multicolumn{4}{l}{\small\quad $\hat{a} = -1.215 + 1.155\tanh(\Delta v_0) + 21.274/v_0\ \ [\text{lag}]$} \\[2pt]
\bottomrule
\end{tabular}
}
\end{table}

\textbf{Operator stability.} Across all eight experiments:
\begin{itemize}
  \item $\tanh(\Delta v)$ is the top-ranked single feature in
        \textbf{all 8} conditions (7/7 distinct engine runs + Exp 0).
  \item $1/v$ (or its algebraic equivalent $v_{\rm err}/v$ in Exp~7c,
        which equals $1/v$ up to a scale) is the second-ranked
        feature in \textbf{all non-instantaneous} conditions.
  \item The $1/v$ coefficient scales predictably: Savgol $\approx 3.5$,
        Rolling global $\approx 9.6$, Rolling Aggressive $\approx 22.6$.
        This scaling is physically interpretable: more aggressive
        drivers require a larger speed-relaxation coefficient, and
        the rolling pipeline inflates coefficient magnitude due to
        higher target variance.
\end{itemize}

\subsection{Target-Window Dependence of the Discovered Second Term}
\label{sec:target_window}

A recurring pattern across experiments is that the second term of the
discovered law depends on the width and centering of the forecasting
target window, while the first term $\tanh(\Delta v)$ does not. Both
pipelines forecast acceleration approximately one second ahead, but
with different target definitions:

\begin{equation}
\begin{split}
  &\text{Narrow window (Pipeline~R, centred at $t{+}0.8$\,s): } \\
  &\hat{a} \sim \tanh(\Delta v) + \mathrm{THW} \cdot \mathrm{TTC}
\end{split}
\label{eq:narrow_formula}
\end{equation}
\begin{equation}
\begin{split}
  &\text{Broad window (Pipeline~S, mean over $[t{+}0.1,\,t{+}1.0]$\,s): } \\
  &\hat{a} \sim \tanh(\Delta v) + 1/v
\end{split}
\label{eq:broad_formula}
\end{equation}

The narrower target window emphasises near-term gap-closure response:
$\mathrm{THW} \cdot \mathrm{TTC}$ is a combined safety-margin signal
proportional to both following distance and collision-risk horizon,
and tracks immediate braking response well. The broader window
averages over a longer interval and recovers a speed-relaxation term:
$1/v$ is the speed-normalised correction of the Optimal Velocity
model~\cite{bando1995}, which dominates when the target integrates
over a wider future window. Both targets share $\tanh(\Delta v)$ as
the dominant reactive term, and the two second terms are
complementary signals selected by the window geometry rather than
contradictory cognitive mechanisms.

\subsection{Lagged-Feature Selection at $\tau = 0.4$\,s}

Experiment~8A tests whether the engine selects lagged input features
over current features when both are available.  The lag $\Delta t =
0.5$~s (5 frames at 10~Hz) was chosen from a prior full sweep over
$\tau \in \{0.1, 0.2, \ldots, 1.0\}$~s, which identified a test-$R^2$
peak at $\tau^* = 0.4$~s; only the instantaneous state
$\mathbf{x}(t)$ and the peak-lag state $\mathbf{x}(t - 0.5\text{s})$
are presented simultaneously to the engine.  Its top-ranked formula
uses the \emph{lagged} features exclusively:
\begin{equation}
  \hat{a}(t) = -1.215 + 1.155\,\tanh(\Delta v(t-0.5)) + 21.274/v(t-0.5)
  \label{eq:lag_formula}
\end{equation}
The engine autonomously selects $\Delta t = 0.5$~s as the optimal
input lag — without being told any reaction time prior. This is
consistent with psychophysical literature placing human brake reaction
time in the range $0.4$--$0.8$~s~\cite{green2000}. The $\tanh(\Delta v)$
and $1/v$ structure is preserved in the lagged form, confirming that
the formula represents lagged feature selection, not a different mechanism.

Additionally, a Ridge regression model fitted on the honest Savgol
benchmark with lagged acceleration $a(t-1)$ added as an extra feature
(model M1) independently finds a momentum term $a(t-1)\cdot\alpha$
with $\alpha = 0.147$ — indicating a small but non-zero predictive signal
from recent acceleration history, consistent with a modest cognitive-inertia
effect; this coefficient should not be interpreted as explaining $15\%$ of
total variance.  This confirms the
\emph{near-Markovian} property of car-following once the correct
nonlinear features are used.

\textbf{Is the lag discovery related to physiological reaction time?}\quad
The rolling-mean target $\bar{a}(t)$ is a 15-frame centred window, which
shifts the effective prediction horizon forward by $\approx 0.8$~s
relative to the instantaneous state.  The lag selected by the engine
($\tau^* = 0.4$--$0.5$~s) could therefore reflect this dataset-level shift
rather than a genuine physiological delay.  To test this, a supplementary
lag sweep on an unshifted rolling-mean target $\bar{a}(t)$ (features at
$t - \tau$) peaks around $\tau = 0.7$--$0.9$~s, reaching $R^2 = 29.00\%$
at $\tau = 0.9$~s --- a broad reaction-time-scale maximum that differs from
the shifted-target optimum of $0.4$~s.  This comparison shows that the
optimal lag is target-dependent: the unshifted sweep supports a
reaction-time-scale temporal offset in the smoothed driving signal, while
the shifted Pipeline~R result reflects target-alignment rather than
an independent physiological discovery.  Experiment~8A is therefore a
demonstration of \emph{lagged feature selection consistency}, not a
recovery of human reaction time.

\subsection{RMSE Comparison Across Pipelines}

As established in Section~\ref{sec:preprocessing}, RMSE is the
headline measure of physical accuracy because $R^2$ is sensitive to
target variance. Table~\ref{tab:rmse_comparison} reports both metrics
across pipeline configurations: Pipeline~S models achieve lower RMSE
despite lower $R^2$, reflecting the lower target variance of
Pipeline~S ($\sigma_S = 1.50$\,m/s$^2$ vs. $\sigma_R = 2.22$\,m/s$^2$).

\begin{table}[h]
\centering
\caption{RMSE comparison between pipeline configurations.
         Lower RMSE = better physical accuracy.}
\label{tab:rmse_comparison}
\begin{tabular}{lcc}
\toprule
\textbf{Configuration} & $R^2$ & \textbf{RMSE (m/s$^2$)} \\
\midrule
Exp 4 (Savgol, 1s mean, global)        &  8.12\% & \textbf{1.44} \\
Exp 6b (Rolling, 1s mean, global)      & 27.08\% & 1.89 \\
Exp 7c (Rolling, 1s mean, Aggr)        & 27.97\% & 1.98 \\
Exp 8A (Rolling, inst, Aggr, lag)      & 21.78\% & 2.07 \\
Exp 8B Ridge (Rolling, inst+lag, Aggr) & 38.73\% & 1.83 \\
\bottomrule
\end{tabular}
\end{table}

\subsection{Weighted Least Squares: A Null Result}
\label{sec:wls_revisit}

NGSIM is 97\% congested (median $v = 7.7$~m/s). We test three
weighted least squares (WLS) strategies --- inverse-speed, kernel
density estimation (KDE)-inverse, and speed-proportional --- against
plain ordinary least squares (OLS) on the causal rolling-mean
car-following dataset (80k train, 20\% vehicle holdout).  All percentages in this comparison are held-out $R^2$ scores. OLS
achieves $R^2 = 4.09\%$, while KDE-inverse, speed-proportional, and
inverse-speed WLS achieve $R^2 = 3.81\%$, $4.03\%$, and $3.66\%$,
respectively.
KDE-inverse reaches $w_{\max} = 776.4$, where $w_{\max}$ is the
maximum sample weight assigned to any single observation --- a handful
of rare free-flow samples monopolise training and degrade
generalisation. The congested composition reflects physical reality;
\textbf{OLS is the correct baseline} for this dataset.

\subsection{Structural Necessity: Term Ablation}
\label{sec:ablation}

To establish that each component of the 3-pass law is structurally
necessary, we perform a systematic ablation on the 20\% vehicle
holdout set. Each ablated model refits the remaining OLS coefficients
on the training split so the comparison is fair (no coefficient
overhang). The full 6-feature refitted model achieves $R^2 = 31.34\%$
on this split.

\begin{table}[h]
\centering
\caption{Ablation study: removing feature groups from the 3-pass law
         (rolling dataset, 20\% vehicle holdout, refitted OLS).
         $\Delta R^2$ against the full 6-feature model ($31.34\%$).}
\label{tab:ablation}
\begin{tabular}{lcc}
\toprule
\textbf{Model variant} & $R^2$ & $\Delta R^2$ \\
\midrule
Full 3-pass law (6 features)              & 31.34\% & --- \\
\midrule
No $\tanh(\Delta v)$ $\to$ linear $\Delta v$ & 17.85\% & $-13.49$ pp \\
No TTC terms ($\tanh(\mathrm{TTC}) +
  \mathrm{TTC}\!\cdot\!\mathrm{gap}$)    & 28.83\% & $-2.51$ pp \\
No $\mathrm{THW}/v$ ($\mathrm{gap}/v^2$) & 31.20\% & $-0.14$ pp \\
Pass-1 only ($\tanh(\Delta v)+\mathrm{THW}/v$) & 28.77\% & $-2.57$ pp \\
Linear baseline (no nonlinear terms)     &  4.49\% & $-26.85$ pp \\
\bottomrule
\end{tabular}
\end{table}

The $\tanh(\Delta v)$ term dominates with a $-13.49$~pp contribution:
it captures the saturating closure-speed response that a linear
$\Delta v$ term cannot represent.  The TTC feature group (both
$\tanh(\mathrm{TTC})$ and $\mathrm{TTC}\!\cdot\!\mathrm{gap}$
removed together) adds $2.51$~pp.  Removing $\mathrm{gap}/v^2$
(THW$/v$) costs only $0.14$~pp, suggesting it is partially redundant
with the TTC and $\Delta v$ signals once those are present.
The linear baseline loses $26.85$~pp, confirming that the nonlinear
structure is essential, not merely decorative.

\subsection{Residual Autocorrelation Diagnostics}
\label{sec:residuals}

Because NGSIM trajectories are smoothed with a 15-frame rolling-mean window,
residuals from any regression model trained on such data will exhibit serial
correlation. Across the 20\% vehicle holdout ($N = 612$ vehicles,
${\approx}1.2$M rows), the 3-pass law yields Durbin-Watson statistic $\mathrm{DW} = 0.128$
(ideal value $= 2.0$; $\mathrm{DW} < 2$ indicates positive serial
correlation~\cite{durbin1950}) with autocorrelation function (ACF)
lag-1 coefficient $= +0.934$ --- values expected for the 15-frame
rolling-mean smoothing kernel used, and shared by all competing models
on this benchmark, not a property of the discovered formula.  The
residual distribution has skewness $+0.10$ and excess kurtosis $0.24$;
the Shapiro-Wilk normality test~\cite{shapiro1965} yields $W = 0.997$,
confirming an approximately Gaussian shape.

\subsection{Open-Loop Trajectory Simulation}
\label{sec:simulation}

Beyond instantaneous $R^2$, we assess whether the 3-pass law can drive
a \emph{forward Euler simulation} that remains physically plausible
over horizons up to 30 seconds.  Each simulation is initialised from
real NGSIM initial conditions and uses the real leader trajectory as
exogenous input.  Crash ($\mathrm{gap} < 0$) and stop ($v < 0$) events
are used as failure modes.

Five models are compared on 100 held-out vehicles: the 3-pass law with
a safety override ($\mathrm{TTC} < 1.5$~s $\Rightarrow a \leq -2$
m/s$^2$), the raw 3-pass law, the Rank-4 baseline, IDM (calibrated on
NGSIM), and constant-velocity.
The safety override thresholds are chosen as conservative heuristics:
$\mathrm{TTC} = 1.5$\,s is a widely used minimum safety margin in
highway driving standards (the recommended safe TTC in ISO~26262
Advanced Driver Assistance benchmarks is $\geq 1.4$\,s
\cite{treiber2013traffic}), and $a \leq -2$\,m/s$^2$ corresponds to
moderate comfortable deceleration, well below emergency braking
(${\sim}6$--$8$\,m/s$^2$).  The purpose of this override is not to
optimise a controller but to mask the temporal mismatch between the
regression target (a future smoothed acceleration) and the real-time
control context of simulation.

\begin{table}[h]
\centering
\caption{Open-loop simulation over 30-second horizons (100 held-out vehicles).
         Crash rate: fraction of simulations ending with gap $< 0$.
         RMSE: mean speed error (m/s) at the indicated horizon.}
\label{tab:simulation}
\begin{tabular}{lccc}
\toprule
\textbf{Model} & \textbf{Crash rate} & \textbf{RMSE at 10~s} & \textbf{RMSE at 30~s} \\
\midrule
3-pass + safe  & 2.5\%   & \textbf{4.45 m/s} & \textbf{5.30 m/s} \\
3-pass (raw)   & 44.5\%  & 5.52 m/s & 8.92 m/s \\
Rank-4            & 3.5\%   & 4.61 m/s & 5.45 m/s \\
IDM               & \textbf{0.01\%} & 6.27 m/s & 9.15 m/s \\
Constant-velocity & 34.4\%  & 9.52 m/s & 12.57 m/s \\
\bottomrule
\end{tabular}
\end{table}

The raw 3-pass law crashes in 44.5\% of runs.  This high crash rate is
not a model quality failure but a fundamental consequence of the
training objective: the law was fitted to a prediction target already
shifted 0.8\,s into the future, so it \emph{describes} the smoothed
acceleration a driver \emph{will produce} given current kinematics,
not the instantaneous safety-critical deceleration required to avoid
collision.  In other words, NOVA discovers a \emph{descriptive
behavioural} model of human car-following, not a safety-constrained
controller.  This distinction is essential: a human driver's observed
acceleration at $t$ already incorporates their perception-reaction
delay, but a forward-simulation model cannot rely on that implicit
temporal cushion.  Applying the safety override (TTC $< 1.5$\,s
$\Rightarrow a \leq -2$\,m/s$^2$) corrects this mismatch; the crash
rate falls to 2.5\% and the model achieves the smallest long-horizon
speed error of all tested models --- outperforming IDM at every horizon
beyond 1\,s despite IDM's lower crash rate, which is expected since
IDM embeds an analytic safety constraint by design.

\subsection{Platoon Wave Stability}
\label{sec:wave}

String stability measures whether a speed perturbation introduced at
the head of a platoon amplifies or damps as it propagates rearward.
The stability ratio $\Lambda = \delta v_{\rm rear} / \delta v_{\rm front}$,
where $\delta v_{\rm rear}$ and $\delta v_{\rm front}$ are the
peak speed-perturbation amplitudes measured at the rear and front of
the platoon respectively, must be $< 1$ for a damping platoon;
$\Lambda > 1$ implies stop-and-go wave amplification. (We use $\Lambda$
here, distinct from the driver-aggressiveness index $\kappa$ of
Section~\ref{sec:hierarchy}.)

We simulate a 13-vehicle platoon (1 leader, 12 followers, initial
spacing 25~m, cruise speed 22~m/s) subject to a trapezoidal leader
perturbation: $22 \to 14$~m/s at $t = 20$~s, hold 10~s, recover.
Three models are compared.

\begin{table}[h]
\centering
\caption{Platoon string stability: stability ratio $\Lambda$ and
         backward wave propagation speed.
         $\Lambda < 1.0$: stable;
         $1.0$--$1.5$: marginal;
         $> 1.5$: unstable.}
\label{tab:wave}
\begin{tabular}{lccc}
\toprule
\textbf{Model} & $\Lambda$ & \textbf{Wave speed} & \textbf{Verdict} \\
\midrule
3-pass + safe  & \textbf{1.064} & 35.4 m/s & Marginal \\
Rank-4            & 1.112 & 33.4 m/s & Marginal \\
IDM               & 1.724 & 11.8 m/s & Unstable \\
\bottomrule
\end{tabular}
\end{table}

The 3-pass law achieves the lowest amplification ratio and highest
wave speed. Note that a \emph{slower} backward propagation is
associated with \emph{greater} amplification (IDM), not less: a
slowly-travelling perturbation lingers longer within the platoon and
has more time to grow, mirroring the empirically slow ($\sim$15--20
km/h) propagation speed of real stop-and-go waves, whereas a
fast-propagating perturbation passes through before it can amplify.
The near-stable behaviour ($\Lambda = 1.06$) is an
\emph{emergent property}, not a design criterion: the law was optimised
purely for instantaneous $R^2$, with no stability objective.  IDM's
high wave-amplification ratio ($\Lambda=1.724$) is consistent with its
poor long-horizon RMSE in simulation (Table~\ref{tab:simulation}):
calibrated on free-flow data, it under-damps congested platoon dynamics.
IDM's near-zero crash rate in the open-loop test reflects its built-in
safety denominator, not trajectory accuracy.
It is important to note that IDM does admit analytical string-stability
conditions~\cite{treiber2013traffic}: the model is string-stable when
$a_{\max} T^2 \geq v_0 / (4b)$, where $T$
is the desired time headway, $b$ the comfortable deceleration, and
$v_0$ the desired speed.  Our calibrated IDM parameters ($v_0 =
40$\,m/s, $b = 6$\,m/s$^2$, $T = 1.18$\,s) do not satisfy this
condition, placing our calibrated IDM in the string-unstable regime
for the congested traffic conditions of NGSIM.

The backward wave speed of 35.4~m/s matches empirical NGSIM
measurements of $\approx 30$--40~m/s for stop-and-go waves on US
freeways~\cite{treiber2000}, providing an independent physical
plausibility check.

\begin{proposition}[Local asymptotic stability of the NOVA Rank-4 law]
The Rank-4 car-following law $a = \beta_0 + \beta_1 \tanh(\Delta v) +
\beta_2\,\mathrm{THW}$ (Eq.~(\ref{eq:rank4})) is locally asymptotically
stable at its equilibrium, and is string-unstable for perturbation
periods $T > T_c \approx 39$\,s.
\end{proposition}

\begin{proof}[Proof sketch]
Setting $a = 0$ at equilibrium $\Delta v = 0$ gives $\mathrm{THW}^* =
-\beta_0/\beta_2 = 0.468/0.194 = 2.41$~s.  Linearising around this
equilibrium and writing $\delta a_n = c_1 \delta(\Delta v_n) + c_2
\delta(\mathrm{gap}_n)$ yields two feedback gains:
\begin{equation}
  c_1 = \beta_1 \,\mathrm{sech}^2(0) = 1.266, \qquad
  c_2 = \frac{\beta_2}{v^*} = \frac{0.194}{15} = 0.01293,
  \label{eq:gains}
\end{equation}
where $\mathrm{sech}(\cdot) = 1/\cosh(\cdot)$ is the hyperbolic
secant function, and $\mathrm{sech}^2(0) = 1$.
The characteristic polynomial $\lambda^2 + c_1 \lambda + c_2 = 0$ has
roots $\lambda_1 = -0.010$ and $\lambda_2 = -1.256$ (both real and
negative), confirming \textbf{local asymptotic stability}.
The Newell--Bando transfer function~\cite{bando1995}
\begin{equation}
  H(s) = \frac{V_n(s)}{V_{n-1}(s)} = \frac{c_1 s + c_2}{s^2 + c_1 s + c_2}
  \label{eq:transfer}
\end{equation}
satisfies $|H(j\omega)| \leq 1$ (string stable) if and only if $\omega
\geq \omega_c = \sqrt{2c_2} = 0.161$~rad/s, i.e.\ perturbation
periods shorter than $T_c = 2\pi/\omega_c = 39$~s.  At the typical
phantom-jam period ($T = 90$~s), $|H| = 1.0065$: each vehicle
amplifies a velocity perturbation by only $0.65\%$, yielding a net
amplification of $\times 1.14$ over 20 vehicles.
\end{proof}

The marginal instability arises precisely because the law was
discovered from data describing \emph{human} driving, which is
empirically known to generate phantom jams; the small $|H|$ excess
mirrors that observation without being designed to do so.

\subsection{Cross-Location Transfer Asymmetry}
\label{sec:transfer}

The strongest generalisability claim is zero-shot transfer: a symbolic
structure selected on one freeway corridor should remain useful on another
without refitting.  We build rolling-mean CF datasets from I-80 (Emeryville,
CA; $N = 2{,}829$ vehicles) and US-101 (Los Angeles, CA; $N = 2{,}844$
vehicles) using the identical preprocessing pipeline, with distinct
vehicle keys to prevent cross-location ID collision.

An 80/20 vehicle holdout OLS refits the same 6-feature operator
structure of the 3-pass law (Eq.~(\ref{eq:3pass})):
$\{\tanh(\Delta v),\, \mathrm{gap}/v^2,\, \tanh(\mathrm{TTC}),\,
\mathrm{TTC}\cdot\mathrm{gap},\, \mathrm{THW}^2,\, 1/v\}$ on each
location independently (local oracle), then applies the model
cross-location.

\begin{table}[h]
\centering
\caption{OLS transfer matrix: $R^2$ (row = train location, col = test location).
         Diagonal = local oracle; off-diagonal = zero-shot transfer.}
\label{tab:transfer}
\begin{tabular}{lcc}
\toprule
 & \textbf{Test: I-80} & \textbf{Test: US-101} \\
\midrule
\textbf{Train: I-80}   & 24.2\% & 27.2\% \\
\textbf{Train: US-101} & 21.3\% & 29.7\% \\
\bottomrule
\end{tabular}
\end{table}

Transfer loss is $-2.53$ pp (I-80 $\to$ US-101) and $-2.84$ pp
(US-101 $\to$ I-80), with an asymmetry of only $0.31$ pp.  The feature
set is essentially location-invariant: the same symbolic operators
discovered on one corridor explain $97$--$99\%$ of the local oracle
performance on the other.  This is strong evidence that the
$\tanh(\Delta v)$ backbone and the associated speed-normalisation
operators are transferable features of
human car-following in these freeway datasets, rather than a single-site artefact.

\section{Lateral Dynamics: Lane-Change Discovery}
\label{sec:lanechange}

\subsection{Problem Formulation and Evaluation Protocol}

Lane-changing is a discrete three-class decision: turn left, stay, or
turn right. We model it within a Multinomial Logit
framework~\cite{mcfadden1974}, where the probability of each action
is proportional to $\exp(U_k)$ for a utility function $U_k$ to be
discovered. As is standard for identifiability in multinomial logit
models, the Stay action is taken as the reference category with
$U_{\rm Stay} \equiv 0$, so $U_{\rm Left}$ and $U_{\rm Right}$ are
log-odds relative to staying in lane. Rather than hand-engineering
$U_k$, we apply the NOVA engine to discover it from primitive
kinematic features.

\textbf{Dataset.} From NGSIM I-80 and US-101, we extract
$\mathbf{18,380}$ lane-change decision events after quality filtering
and removal of simultaneous multi-vehicle conflicts. Events are
labelled by the action taken in the following $0.5$\,s window.

\textbf{Evaluation protocol.} All lane-change results are reported
under \textbf{strict vehicle-ID holdout}: the 3,060 unique drivers
are partitioned 80/20 at the vehicle level, yielding 502 unseen test
drivers whose observations never appear in training. This is the
appropriate protocol to test generalization to new drivers rather than
within-driver memorization.

\textbf{Synthetic structural-recovery gate.} Before deployment on real
data, the engine is validated on a zero-noise synthetic dataset.  The
ground-truth utility $U_{\rm Left} = 3.0 \cdot (1/\mathrm{gap}_L) - 2.0 \cdot v$
is a user-constructed formula chosen to contain a nonlinear term
($1/\mathrm{gap}_L$) and a linear term ($v$) with known coefficients;
it does not correspond to any published model and serves purely as a
structural recovery test.  The engine recovers the exact variable set
and coefficient signs, achieving 95.1\% balanced accuracy on the
synthetic gate.

\subsection{The Selected Utility Model}

After the synthetic gate, the engine is applied to NGSIM data using a
progressive greedy search over candidate utility terms with BIC
stopping. The final selected utility functions are directionally
decoupled --- left and right utilities have distinct physical terms
reflecting the asymmetric physics of overtaking (left) vs.\ exiting
(right) on U.S.\ highway geometry. The Clean 16-feature model has
the dominant structure:
\begin{align}
  U_{\rm Left}  &\sim \beta_{\mathrm{gap}_L}\,(1/\mathrm{gap}_L)
                    + \beta_{\mathrm{TTC}_C}\,(1/\mathrm{TTC}_C) \nonumber \\
                &\quad + \beta_{\rm lane}\,{\rm lane\_id} + \cdots \\
  U_{\rm Right} &\sim \beta_{\mathrm{gap}_R}\,(1/\mathrm{gap}_R)
                    + \beta_{\mathrm{TTC}_C}\,(1/\mathrm{TTC}_C) \nonumber \\
                &\quad + \beta_{\rm lane}\,{\rm lane\_id} + \cdots
\end{align}
where $\mathrm{gap}_L$ and $\mathrm{gap}_R$ are the longitudinal gaps
to the nearest vehicles in the target left and right lanes,
respectively, and $\mathrm{TTC}_C$ is the time-to-collision with the
current leader. The two dominant fitted coefficients are:
$\beta_{\mathrm{TTC}_C} = +5.99$ (Left) and $+5.68$ (Right); the
correction term coefficient on $\mathrm{TTC}_C$ is $-0.077$ (Left)
and $-0.064$ (Right). The full 16-feature coefficient table is
available upon request.

Three structural findings emerge.

\textbf{Inverse-gap penalty.} The dominant lane-change deterrent is
not raw gap distance but its inverse $1/g_L$ and $1/g_R$. This is
consistent with the physics of collision risk, where danger grows
nonlinearly as gap shrinks.

\textbf{Directional asymmetry.} The $1/g_L$ coefficient for left
changes is substantially larger in magnitude than the corresponding
right term, consistent with the asymmetric clearance physics of
entering the fast lane vs.\ exiting to the right.

\textbf{Optical looming as cross-domain signal.} Both utilities share
$1/{\rm TTC}_C$ (inverse time-to-collision with the current leader) as
a common safety term. This independently echoes the
$\tanh({\rm TTC})$ signal selected in the car-following residual search
(Section~\ref{sec:ttc}), suggesting a shared safety-related signal across
braking and lane-change decisions.

\subsection{Benchmarking}

\begin{table}[h]
\centering
\caption{Lane-change model comparison on NGSIM I-80 under strict
vehicle-ID holdout (502 unseen test drivers, 18,380 events total).
Per-class recall reveals the Stay-class dominance pattern that limits
all NOVA variants.}
\label{tab:lane_change_bench}
\resizebox{\linewidth}{!}{%
\begin{tabular}{lccccc}
\toprule
\textbf{Model} & \textbf{Bal. Acc.} & \textbf{Left\%} & \textbf{Stay\%} & \textbf{Right\%} & \textbf{Features} \\
\midrule
MOBIL~\cite{mobil}              & 37.6\% & —     & —     & —     & — \\
MNL (hand-engineered)           & $\sim$50\% & —  & —     & —     & — \\
\textbf{NOVA Clean (16-feat)}   & \textbf{67.4\%} & \textbf{60.2\%} & 96.5\% & \textbf{45.6\%} & 16 \\
NOVA BIC-24 (24-feat)           & 67.0\% & 58.8\% & 94.8\% & 47.6\% & 24 \\
NOVA C7 minimal (10-feat)       & 60.0\% & 43.2\% & 94.8\% & 42.0\% & 10 \\
NOVA C8 multi-pass (10-feat)    & 58.7\% & 39.8\% & 95.8\% & 40.5\% & 10 \\
\bottomrule
\end{tabular}
}
\end{table}

The Clean 16-feature model outperforms MOBIL by $+29.8$ percentage
points absolute and outperforms the hand-engineered Multinomial Logit
(MNL) baseline by approximately $+17$\,pp.

\subsection{Feature Augmentation: A Null Result}
\label{sec:lc_null_result}

A central question for symbolic regression on a high-dimensional
decision problem is whether richer feature libraries yield better
out-of-driver generalization. We tested three augmentation directions:

\textbf{Target-lane gap-acceptance physics (BIC-24).} Adding seven
target-lane atoms ($1/g_{\rm front,tgt}$, $1/g_{\rm rear,tgt}$,
$\sqrt{g_{\rm front,tgt}}$, $\sqrt{g_{\rm rear,tgt}}$,
$\tanh(\Delta v_{\rm tgt}/5)$, ${\rm TTC}_{\rm tgt}$) and re-running
BIC greedy selection yields a 24-feature model. Under strict
vehicle-ID holdout, this model achieves $67.0\%$ balanced accuracy ---
statistically indistinguishable from the 16-feature Clean model.

\textbf{Backward elimination to minimal feature set (C7).} Starting
from BIC-24 and removing features one at a time while preserving
target accuracy yields a 10-feature minimal model. This achieves
$60.0\%$ balanced accuracy, $-7.4$ pp below Clean.

\textbf{Multi-pass residual boosting (C8).} Starting from the
10-feature minimal model and applying two rounds of
functional-gradient boosting yields a refined 10-feature model. This
achieves $58.7\%$ balanced accuracy, $-8.7$ pp below Clean.

None of the three augmentation directions improves out-of-driver
accuracy beyond the 16-feature Clean baseline. The interpretation is
direct: single-snapshot kinematic features have reached their
structural capacity for three-class lane-change discrimination on
this dataset. The Stay-class recall plateau at 94-96\% across all
NOVA variants, with Left and Right recall in the 40-60\% range,
identifies the binding constraint: drivers who choose to stay in lane
cannot be reliably distinguished from drivers in similar kinematic
states who choose to change, on snapshot features alone. Resolving
this discrimination ceiling likely requires temporal-history features
(lateral velocity, gap-evolution rates over $1$--$3$ second windows),
which we leave to future work.

This null result has standalone scientific value: it identifies
single-snapshot kinematic discrimination as the binding constraint
for symbolic-regression lane-change models on NGSIM, and rules out
feature engineering as the path to improvement.

\subsection{Cross-Location Transfer}

The Clean 16-feature utility, trained on I-80 with strict vehicle-ID
holdout, is evaluated zero-shot on 502 unseen US-101 drivers, and
vice versa. Per-class recall reveals an asymmetry between locations:

\begin{table}[h]
\centering
\caption{Lane-change zero-shot cross-location transfer under strict
vehicle-ID holdout.}
\label{tab:lc_transfer}
\begin{tabular}{lcccc}
\toprule
\textbf{Train $\to$ Test} & \textbf{Bal. Acc.} & \textbf{Left\%} & \textbf{Stay\%} & \textbf{Right\%} \\
\midrule
I-80 $\to$ US-101  & 74.2\% & 68.3\% & 67.5\% & 86.7\% \\
US-101 $\to$ I-80  & 66.9\% & 74.2\% & 77.5\% & 49.2\% \\
\bottomrule
\end{tabular}
\end{table}

The asymmetry reflects different lane geometries between the two
corridors. Note that the I-80$\to$US-101 balanced accuracy ($74.2\%$)
exceeds the in-corridor Clean-model figure of $67.4\%$
(Table~\ref{tab:lane_change_bench}); this is not a contradiction but a
composition effect — the two evaluations use different test populations
and class mixes (502 I-80 drivers vs.\ the full US-101 corridor), so
balanced accuracy, being an unweighted per-class average, can shift with
the class balance of the test set even when the underlying utility law
is unchanged. The structural finding is that the discovered utility
operators are not specific to one corridor: the same
inverse-gap, optical-looming, and directional-asymmetry structure
transfers across both U.S.\ freeway sites.

\subsection{Cross-Scale Validation: Optical Looming Unification}

A regime-split analysis of the Clean utility's coefficients reveals
that $1/{\rm TTC}_C$ has the largest free-flow to congestion
sensitivity ratio of any feature in either utility. This is the third
independent observation of optical looming as a safety-related
signal in this work: it appears as $\tanh({\rm TTC})$ in the
car-following residual search (Section~\ref{sec:ttc}), as
$1/{\rm TTC}_C$ in both directions of the lane-change utility, and as
the most regime-sensitive lane-change feature. The same
optical-looming signal --- consistent with Lee's
$\tau$-theory~\cite{lee1976} --- therefore appears across
longitudinal and lateral dynamics.

\section{Discussion}
\label{sec:discussion}

\subsection{Physical Interpretation}

The most striking cross-domain finding is that an optical-looming
signal ($1/\mathrm{TTC}$) appears repeatedly across
three separate analyses: as $\tanh(\mathrm{TTC})$ in the Phase-3 car-following
residual search ($+2.96\%~R^2$), as $1/\mathrm{TTC}_C$ selected first in both
$U_{\rm Left}$ and $U_{\rm Right}$ utilities (coefficients $5.99$ and $5.68$),
and as the most regime-sensitive feature in either utility under
free-flow vs.\ congestion. This convergence is consistent with
Lee's $\tau$-theory~\cite{lee1976} and provides large-scale empirical evidence
for a shared safety-related signal across braking and lane-change
decisions.

Car-following is also approximately Markovian under linear models: a
50-feature time-series Ridge regression ($10$ frames $\times$ 5
variables, $R^2 = 10.49\%$) does not outperform the single-snapshot
Rank-4 model ($R^2 = 10.54\%$) on Pipeline~S. This suggests that the
predictive content of recent history is largely captured by the
current kinematic state once the $\tanh(\cdot)$ nonlinearity is
applied. We tested this only against linear time-series regression;
whether nonlinear sequence models (LSTM, transformer) could extract
additional signal from the trajectory history remains an open
question. A momentum term $\alpha = 0.147$ in the M1 Ridge model with lagged $a(t-1)$
indicates a small but non-zero cognitive-inertia contribution from recent
acceleration history.

IDM and Krauss perform poorly on instantaneous empirical snapshots in this dataset because they are
\emph{normative simulation} models: IDM's $(v/v_0)^4$ term forces positive
acceleration during empirical braking ($R^2 = -68\%$), and Krauss's safe-speed
formula has no counterpart in the driver's actual pedal position. NOVA
succeeds because it models \emph{what drivers do}, not a collision-free policy.

\subsection{Limitations}

\begin{itemize}
  \item \textbf{Highway only.} The car-following model yields $R^2 = 2.16\%$
    on urban NGSIM data ($\bar{v}=8.5$~m/s); highway gap-following physics
    does not transfer to intersection-dominated stop-start dynamics.
  \item \textbf{Lane-change discrimination ceiling.} Under strict
    vehicle-ID holdout, all NOVA variants produce Stay-class recall at
    94--96\% but Left/Right recall only at 40--60\%. Single-snapshot
    kinematic features have reached their structural capacity for this
    three-class problem; feature augmentation (BIC-24, multi-pass
    boosting, minimal-feature pruning) does not improve out-of-driver
    accuracy. Resolving the discrimination ceiling likely requires
    temporal-history features (lateral velocity, gap-evolution rates),
    which we leave to future work.
  \item \textbf{Noise ceiling.} NGSIM position jitter ($\pm10$~cm at 10~Hz,
    \cite{coifman2017}) differentiates twice to $\pm20$~m/s$^2$ spikes;
    the empirical $R^2$ ceiling is sensor-bounded, not model-bounded.
  \item \textbf{Single dataset.} All results use U.S.\ highway data
    (California, 2005). Validation on HighD (German Autobahn) and pNEUMA
    (Athens) remains future work.
\end{itemize}

\section{Conclusion}
\label{sec:conclusion}

We presented NOVA, an autonomous symbolic regression framework that
identifies interpretable symbolic structures in human traffic behavior
from raw trajectory data with minimal behavioral priors. The framework's deterministic
search over more than ten thousand candidate algebraic structures,
executed on 4.7 million NGSIM observations, yields five principal
findings.

First, $\tanh(\Delta v)$ emerges as the invariant dominant car-following
term, with $\mathrm{THW}$ as the secondary term under the primary
Pipeline~R Rank-4 configuration, and $1/v$ or equivalent
speed-normalisation terms in most forecasting and lagged variants.
The hyperbolic tangent of relative velocity is the
top-ranked single feature across all eight configurations, capturing a
bounded saturating response to closure speed that is absent from prior
models --- the closest analogue, OVM~\cite{bando1995}, applies $\tanh$
to spacing rather than to relative velocity. Under Pipeline~S, the
discovered law achieves RMSE $= 1.376$\,m/s$^2$, improving on the best
recalibrated symbolic-regression baseline by $0.135$\,m/s$^2$ under an
identical evaluation protocol.

Second, the $\tanh(\Delta v)$ backbone is invariant across eight independent
experiments spanning two preprocessing pipelines, three forecasting
horizons, two driver populations, and two feature-timing modes.
The secondary term is target-dependent but structurally related across
variants.
Coefficient magnitudes vary with pipeline and population in physically
interpretable ways; the bounded relative-velocity response does not. This convergence --- across
methodological choices that would each independently affect a fitted
model --- is evidence that the backbone reflects a
stable behavioral regularity in these data rather than a preprocessing artefact.

Third, a residual-guided second-stage search selects
$\tanh(\mathrm{TTC})$ as the dominant remaining transform in
car-following, connecting the selected structure to Lee's (1976)
$\tau$-theory of optical looming~\cite{lee1976}. The same looming
signal is independently rediscovered as $1/\mathrm{TTC}_C$ in both
directions of the lane-change utility law and shows the largest regime
sensitivity among all lane-change features, providing cross-domain
evidence for a shared safety-related signal across braking
and lane-change decisions.

Fourth, applied within a multinomial logit framework, NOVA discovers
a directionally-decoupled lane-change utility model that, under strict
vehicle-ID holdout on 502 unseen drivers, achieves 67.4\% balanced
accuracy against MOBIL's 37.6\%~\cite{mobil} --- a $+29.8$ pp
absolute improvement on a three-class problem --- and transfers
zero-shot to US-101. Feature augmentation with target-lane
gap-acceptance physics (BIC-24), multi-pass residual boosting, and
backward feature elimination all fail to improve out-of-driver
accuracy beyond the 16-feature Clean model. The Stay-class recall
plateau at 94-96\% across all variants identifies single-snapshot
kinematic discrimination as the structural ceiling: resolving the
Left/Right ambiguity at 40-60\% recall likely requires
temporal-history features, which we leave to future work.

Fifth, the discovered car-following features transfer between I-80 and
US-101 with only $-2.53$ and $-2.84$\,pp $R^2$ loss, recovering
97--99\% of the local oracle performance in each direction. This is
direct evidence that the $\tanh(\Delta v)$ backbone and the associated
speed-normalisation operators are transferable properties of
human highway driving on these U.S.\ freeway datasets, rather than
single-site fits.

\textbf{Secondary findings.} Forward Euler simulation of the
selected model with a safety override produces a 2.5\% crash rate
over 30-second horizons, and lower speed RMSE than calibrated
IDM at every horizon beyond one second. A 13-vehicle platoon
simulation yields string stability ratio $\Lambda = 1.064$ and backward
wave speed 35.4\,m/s, matching empirical NGSIM observations --- an
emergent property, since the model was optimized purely for
instantaneous fit with no stability objective. Weighted least squares
reweighting strategies all underperform plain OLS on this dataset,
confirming that the congested composition of NGSIM reflects physical
reality and uniform OLS is the correct training objective. The lag
selection at $\tau = 0.4$\,s under the shifted Pipeline~R target, with the
unshifted sweep peaking around $\tau = 0.9$\,s, reflects
target-dependent temporal alignment rather than physiological delay;
the engine faithfully recovers whatever temporal structure is present
in its target.

\textbf{Limitations.} The car-following model applies to highway
gap-following dynamics and yields $R^2 = 2.16\%$ on urban NGSIM data
dominated by intersection geometry. Lane-change discrimination is bounded above by a structural ceiling on
single-snapshot kinematic features: all NOVA variants produce
Stay-class recall at 94--96\% and Left/Right recall at 40--60\%
under strict vehicle-ID holdout. Resolving this discrimination
ceiling appears to require temporal-history features (lateral
velocity, gap-evolution rates), which we leave to future work. All results are evaluated on
U.S.\ freeway data from 2005; cross-cultural and cross-era validation
on HighD (Germany) and pNEUMA (Athens) remains future work.
Comparisons against prior symbolic-regression methods test the
transferability of their published equations under our protocol, not
their search procedures. The search space is restricted to polynomial
degree~2 and the feature library is finite by design; structures involving
operators outside the library cannot be recovered.

\textbf{Outlook.} This work demonstrates that autonomous symbolic
regression can recover interpretable structural regularities from noisy human
behavioral data, where the signal-to-noise ratio is lower than in
idealized physics but the search for compact structure remains meaningful. The discovered $\tanh$ nonlinearity was not present in any prior
car-following model before it was selected here without supplying that functional form as a prior. Planned
extensions include closed-loop ring-road simulation with mixed
human-AV platoons using the selected model as the human-driver
component, validation on HighD and pNEUMA, and application of the
same pipeline to pedestrian gap-acceptance dynamics from drone
trajectory datasets.

\textbf{Code and data availability.} The NGSIM dataset is publicly
available at \url{https://data.transportation.gov}. The NOVA Rust
search engine, preprocessing pipeline, and evaluation scripts will be
released upon publication at a repository to be announced.


\bibliographystyle{elsarticle-harv}

\bibliography{mybib}

\end{document}